\documentclass[runningheads]{llncs}
\usepackage[T1]{fontenc}
\usepackage{graphicx}
\usepackage{xcolor}
\usepackage[ruled, vlined, linesnumbered]{algorithm2e}
\usepackage{amsmath}
\usepackage{amssymb}
\usepackage{mathtools}
\usepackage[caption=false,font=footnotesize]{subfig}

\SetCommentSty{mycommfont}

\begin{document}
\title{Scalable Multi-Agent Maze Traversal \\ with Local Communication}
%
%
\author{Julian Rau\inst{1}\orcidID{0009-0006-0913-6555} \and
    Jahir Argote-Gerald\inst{2}\orcidID{0009-0003-1492-1329} \and
    Grace McFassel\inst{1}\orcidID{0000-0002-8568-1313} \and 
    Genki Miyauchi\inst{3}\orcidID{0000-0002-3349-6765} \and
    Paul Trodden\inst{2}\orcidID{0000-0002-8787-7432} \and
    Roderich Groß\inst{1,2}\orcidID{0000-0003-1826-1375}
}
\authorrunning{J. Rau et al.}
%
\institute{Department of Computer Science, Technical University of Darmstadt, Germany {\tt\small \{julian.rau, grace.mcfassel, roderich.gross\}@tu-darmstadt.de}\and
    School of Electrical and Electronic Engineering, The University of Sheffield, UK {\tt\small \{jaargotegerald1,  p.trodden\}@sheffield.ac.uk}\and
    Bristol Robotics Laboratory, University of Bristol, Bristol, UK {\tt\small g.miyauchi@bristol.ac.uk}}
\maketitle              
\begin{abstract}
Cave networks, pipe systems, and similar maze-like environments pose significant challenges for multi-agent navigation in unknown settings with limited communication. We propose a distributed algorithm that enables agents to collectively traverse an unknown, possibly cyclic graph. Agents enter sequentially at a designated start node and are tasked to localize and reach an undisclosed goal while avoiding collisions. They coordinate via local communication using leader–follower relationships and leader switching. At any moment in time, exploration is performed by only one of the agents, which runs a single-agent maze solver. We prove that the algorithm is complete, that its makespan is asymptotically equivalent (in the number of agents) to that of an optimal full-knowledge strategy, and derive its time and space complexity. Simulations with up to $625$ agents show a decreasing average sum-of-fuels as the number of agents increases and demonstrate that the proposed approach outperforms a naïve baseline in which all agents independently execute the single-agent solver.

\end{abstract}

\section{Introduction}
\label{sec:introduction}
Many real-world scenarios require multiple agents to navigate efficiently through complex, maze-like environments such as networks of pipes~\cite{parrott2020simulation,ZHANG2023pipeinspection} or caves~\cite{murphy2009mobile,martz2020survey,caveexploration}.
Coordinating such agents poses unique challenges, including 
the uncertainty of environments,
communication restrictions, 
and inter-agent collisions.

Environments are often modeled as graphs. If the graph is unknown a priori, its nodes and edges are locally revealed as agents traverse it. For single-agent systems, classical methods such as breadth-first search (BFS) and Tr\'emaux's algorithm can be employed to search the graph~\cite{randomwalkruntime,cormen2022introduction,even2011graph}.
For multi-agent system, inter-agent collisions and communication have to be considered. Some studies consider agents that explore (tree or) graph environments while coordinating through beacons they place at  nodes~\cite{BrassMultiExploration,cabrera2012flooding,fraigniaud2006collective,nagavarapu2016multi} or through local or global communication~\cite{dereniowski2015fast,linardakis2024}. Except for~\cite{linardakis2024}, these studies provide theoretical guarantees such as completeness and exploration-time bounds but do not all consider inter-agent collisions~\cite{BrassMultiExploration,dereniowski2015fast,fraigniaud2006collective}. Moreover, their focus is on exhaustive exploration rather than reaching specific goal nodes. The search algorithm proposed in~\cite{crnkovic2023fast} aborts in the event undisclosed goal nodes are discovered. However, it is centralized.

A related problem is Multi-Agent Path Finding (MAPF), which involves identifying collision-free paths of multiple agents to designated goal nodes in a known graph environment~\cite{sharon2015conflict,silver2005cooperative,SternSoCS19,lavalleMAPF}. In classical MAPF settings, agents are assumed to have full knowledge of the graph, their own start and goal nodes, and those of other agents. Planning is typically conducted in a centralized manner~\cite{sharon2015conflict,silver2005cooperative,lavalleMAPF}. Some MAPF variants relax these assumptions by requiring agents to discover parts of their environment or plan their paths in a distributed manner~\cite{hall2020self,nebel2019implicitly,queffelec2023complexity,shofer2023multi}, but do not permit fully unknown environments.

Multi-Agent Maze Traversal (MAMT)~\cite{kivelevitch2010multi,Jahir} is a problem that combines multiple of the aforementioned challenges. Specifically, it concerns collective, collision-free navigation of arbitrary large group of agents (i) through a unknown graph environment that is only gradually revealed through local sensing as the agents move, and (ii) towards a common, undisclosed goal node.
In~\cite{kivelevitch2010multi}, a distributed approach for MAMT is presented but 
assumes global communication among agents. In~\cite{Jahir}, we proposed a distributed approach requiring only local communication. It employs a leader-switching strategy 
by which one agent (the \emph{head}) executes a single-agent maze solver to search for the goal, while all other agents choose neighboring agents to follow, thereby 
directly or indirectly following the head. If movement were to result in a collision, the head instead passes its role to the respective agent.
Both approaches are limited to acyclic graph environments (i.e., trees) and offer no theoretical guarantees~\cite{kivelevitch2010multi,Jahir}, leaving decentralized, locally coordinated MAMT for general graphs an open problem.

The contributions of this work are as follows: (i) We propose a new MAMT algorithm, based on the one presented in~\cite{Jahir}, to address the MAMT problem for general (i.e. possibly cyclic) graphs. The algorithm manages competition between agents and maintains connected communication within the swarm. It employs multiple messaging rounds per time step and diverse message types. (ii) We formally prove completeness and that the makespan is asymptotically equivalent (with respect to the number of agents) to the one achieved by an optimal strategy assuming full knowledge.
(iii) We derive the algorithm's time and space complexity.
(iv) We validate the algorithm in simulation with up to $625$ agents, showing a decrease in average sum of fuel with growing numbers of agents, and that the algorithm 
outperforms a naïve baseline (where all agents run the single-agent maze solver) and approaches the performance of the full-knowledge strategy.

Sections~\ref{sec:problemformulation} and~\ref{sec:algorithm} define the MAMT problem and 
propose and analyze an algorithmic solution.
Sections~\ref{sec:results} and~\ref{sec:conclusion} present simulation results and conclusions.

\section{Problem Formulation}
\label{sec:problemformulation}
\begin{figure}[t]
  \centering
  \subfloat[]{
        \includegraphics[width=.27\columnwidth]{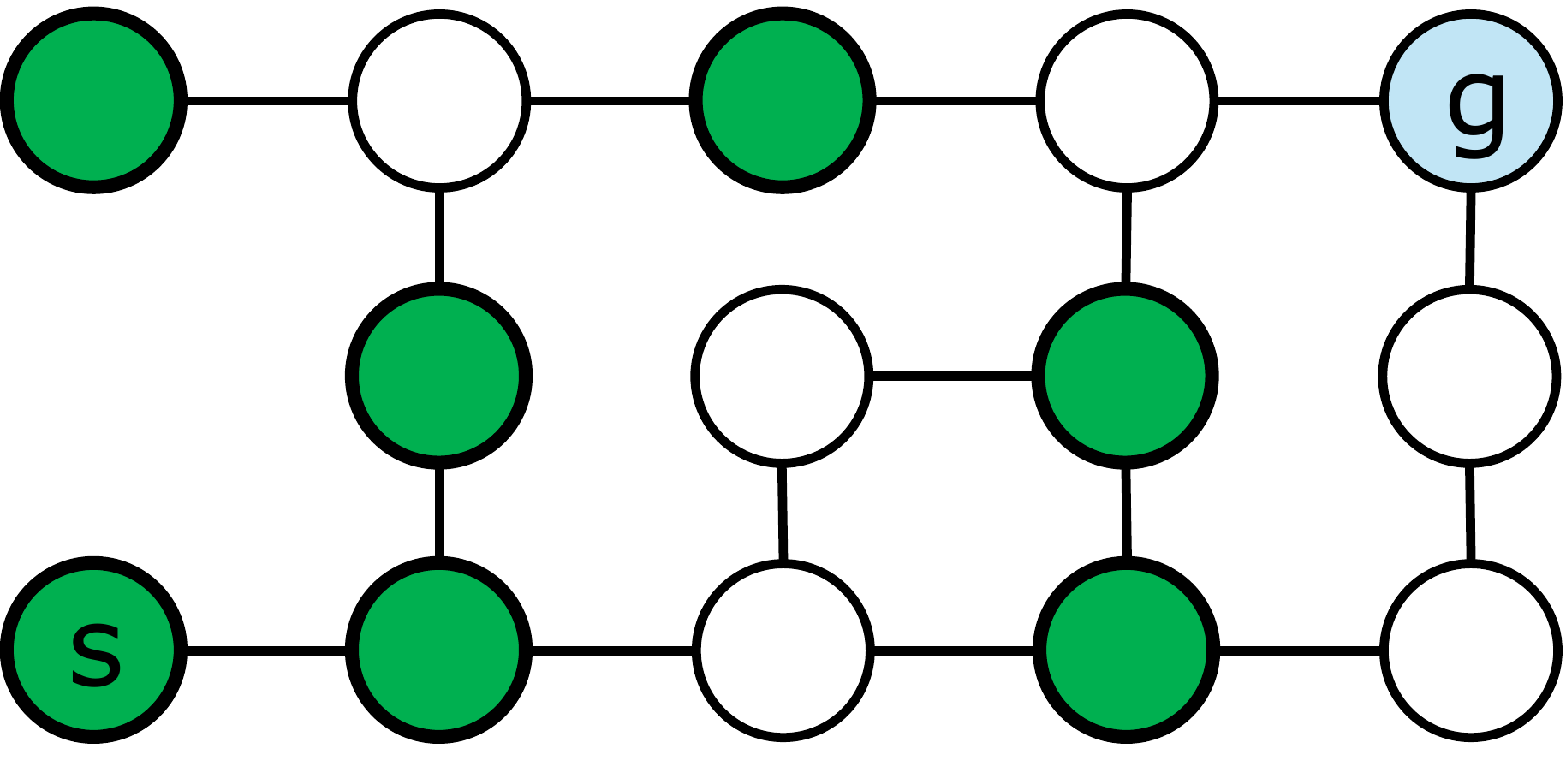}
        \label{fig:mazeenv}
    }\hfill
    \subfloat[]{
        \includegraphics[width=.27\columnwidth]{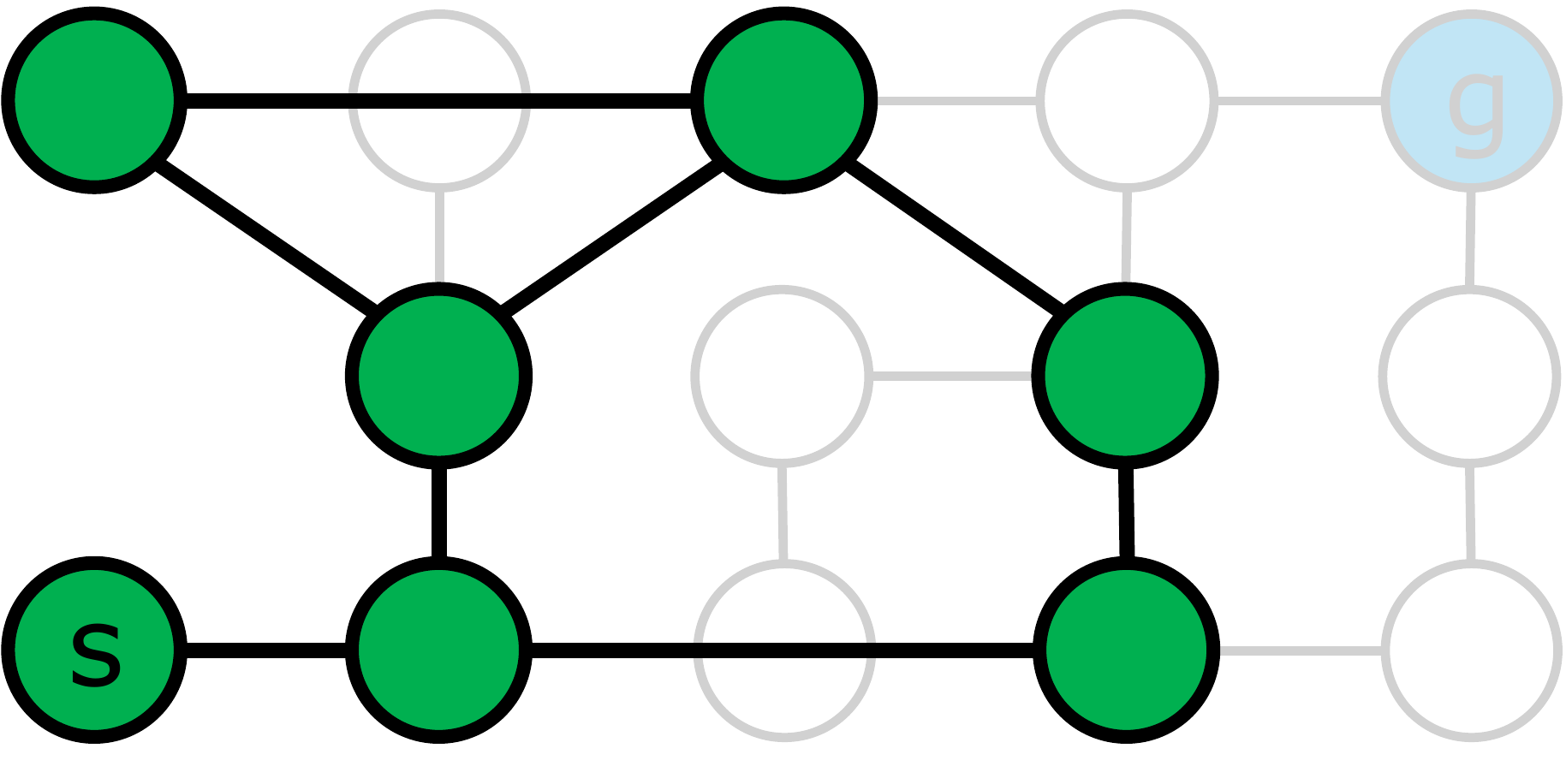}
        \label{fig:mazecoms}
    }\hfill
    \subfloat[]{
        \includegraphics[width=.27\columnwidth]{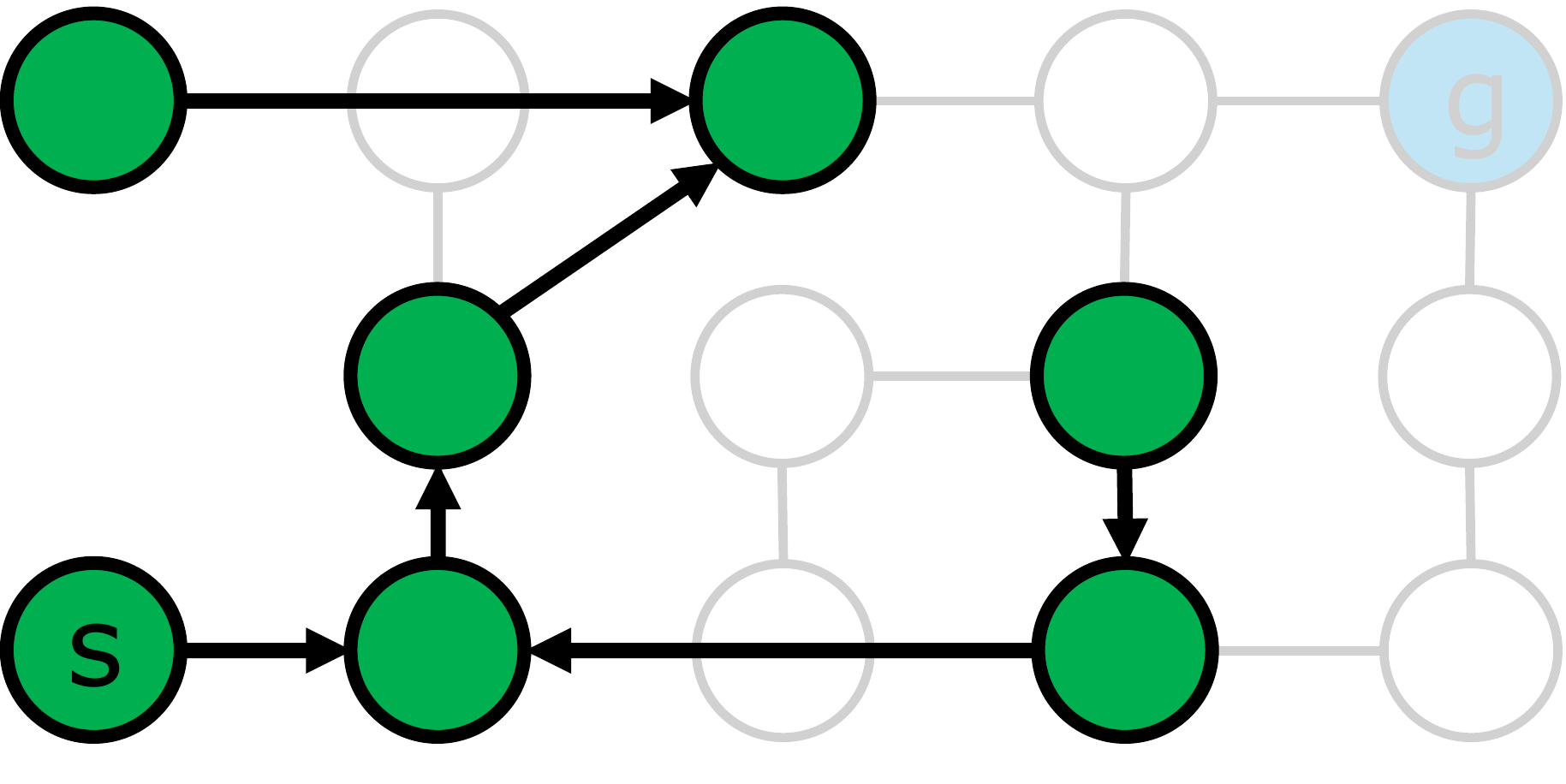}
        \label{fig:mazelead}
    }
  \caption{(a) The maze is represented as a graph with start node $s$, goal node $g$ (blue). Agents (green) are tasked with moving from $s$ to $g$ without prior knowledge of the environment. (b) Agents can only communicate with other agents that are either in an adjacent node or two nodes away with an empty node in between. (c) All agents but one choose a leader to follow, resulting in a leader-follower relation. 
  }
  \label{fig:problemformulation}
\end{figure}
The maze is represented as a connected, undirected graph $\mathcal{G} = (\mathcal{V}, \mathcal{E})$, where 
$\mathcal{V}$ is the set of nodes and $\mathcal{E}$ is the set of edges. An example graph is shown in Fig.~\ref{fig:mazeenv}. Two nodes in this graph, $s,g \in \mathcal{V}$, $s \neq g$, represent the start and the goal (blue), respectively. Node $s$ is assumed to be a leaf in $\mathcal{G}$\footnote{As graphs representing real-world environments are artificially created, we assume they can be constructed accordingly.}.

A set of $n$ agents, $\mathcal{A} = \{1, \dots , n\}$, is considered, where $v_i[k]\in \mathcal{V}$ denotes the node that agent $i\in\mathcal{A}$ resides on at time $k \in \mathbb{N}_0$. We define time step $k + 1$ as the interval $[k,k+1)$. 
For simplicity, we omit $k$ when it is clear from the context.
At time $k=0$, all agents are at the start (i.e., $\forall i: v_i[0]=s$) and one agent is considered \emph{active} whereas all others are \emph{inactive}. Active agents remain active indefinitely. 
At the precise time an active agent leaves the start node, an inactive agent present at the same node (if any) becomes active.
Activation repeats until the last agent becomes active and occurs in a fixed, predetermined total order $\leq_\mathcal{A}$ on $\mathcal{A}$. Let $\mathcal{A}^{\text{active}}[k] \subseteq \mathcal{A}$ denote the set of active agents at time $k$.

Let $\mathcal{N}_i[k] = \left\{u \in \mathcal{V} \mid \{v_i,u\} \in \mathcal{E}\right\}$ be the set of adjacent nodes of agent $i$.
In time step $k+1$, agent $i$ can either move to a node in $\mathcal{N}_i[k]$ or remain in its current node $v_i[k]$. 
A node $v \in \mathcal{V}\setminus\{g\}$ is occupied at time $k$, if $\exists i \in \mathcal{A}^{\text{active}}[k]: \, v_i[k]=v$.
A \emph{vertex conflict}~\cite{SternSoCS19} occurs at time $k$ if two active agents occupy the same non-goal node, that is, $\exists i,j\in \mathcal{A}^{\text{active}}[k],i \neq j: v_i[k] = v_j[k]
\neq g$.  
A \emph{following conflict}~\cite{SternSoCS19} occurs if an active agent moves to a non-goal node in the same time step that another active agent left it, formally, $\exists i,j\in \mathcal{A}^{\text{active}}[k]$, $i \neq j: v_i[k] = v_j[k+1] \neq g$. Both vertex and following conflicts are prohibited.
Note that an agent $i$ that activates in timestep $k+1$ is not yet active at $k$ (i.e., $i \notin \mathcal{A}^{\text{active}}[k]$) and therefore does not induce a following conflict with the active agent $j \in \mathcal{A}^{\text{active}}[k]$ that left the start node in time step $k+1$. 

At time $k$, agent $i$ knows\footnote{For ease of notation, we write  $\mathcal{N}_i$ as if agent $i$ had access to the global labels of neighboring nodes, $\mathcal{V}$. However, agent $i$ can only distinguish these nodes, requiring labels that are consistent from agent $i$'s local perspective.} (adapted from~\cite{Jahir}):
\begin{itemize}
    \item its unique \emph{ID} $i$
    \item whether the node $v_i[k]$ it currently resides on is the  
    goal node, $v_i[k]=g$
    \item if $g$ is an adjacent node, that is, if $g \in \mathcal{N}_i[k]$, and if so, which edge leads to $g$
    \item the subset of adjacent nodes that are currently occupied, $\mathcal{N}_i^{\text{occupied}}[k] = \left\{u \in \mathcal{N}_i[k]\setminus \{g\} \mid \exists a \in \mathcal{A}^{\text{active}}[k]\colon v_a[k] = u\right\}$. The subset of unoccupied adjacent nodes can be obtained by $\mathcal{N}_i^{\text{unoccupied}} = \mathcal{N}_i \setminus \mathcal{N}_i^{\text{occupied}}$.
\end{itemize}

Messages among agents are limited to a two-hop distance and can only pass unoccupied nodes\footnote{This models constrained local communication, such as line-of-sight communication.} (example shown in Fig.~\ref{fig:problemformulation}b), resulting in the undirected communication graph $\mathcal{G}^{\text{com}}=(\mathcal{V}^{\text{com}}, \mathcal{E}^{\text{com}})$, where $\mathcal{V}^{\text{com}} = \mathcal{A}^{\text{active}}$ and \cite{Jahir}\footnotemark
\begin{equation*}
\begin{aligned}
     \mathcal{E}^{\text{com}} =& \Big\{\{i,j\} \subseteq \mathcal{A}^{\text{active}} \mid i \neq j \land \\
     & \left(v_i = v_j \lor \{v_i, v_j\} \in \mathcal{E} 
     \lor \left(\mathcal{N}_i^{\text{unoccupied}} \cap \mathcal{N}_j^{\text{unoccupied}} \neq \emptyset\right)\right)\Big\}.
\end{aligned}
\end{equation*}
Consequently, agent $i$ can communicate with agents in its communication range $\mathcal{C}_i = \left\{j \in \mathcal{V}^{\text{com}}\setminus \{i\} \mid \{i,j\} \in \mathcal{E}^{\text{com}}\right\}$~\cite{Jahir}.
\footnotetext{The expression $v_i = v_j$ in $\mathcal{E}^{\text{com}}$ covers the case where two active agents reside in $g$.}
Agents can send a message via directed cast, that is, via a specified edge.
Formally, when agent $i$ sends a message via directed cast along edge $\left\{v_i, u\right\} \in \mathcal{E}$, $u \in \mathcal{N}_i$, this message reaches all agents that currently reside in any node of the set 
\begin{equation*}
    \mathcal{U}_i[u]\coloneq\left\{w \in \mathcal{V} \mid w = u \lor \left(u \in \mathcal{N}_i^{\text{unoccupied}} \land \{u,w\} \in \mathcal{E} \right)\right\}.
\end{equation*}
Agents can also send a message via (local) broadcast, that is, along all outgoing edges of their current node~\cite{Jahir}. A broadcast message sent by agent $i$ reaches all agents currently residing in any node of set 
$\mathcal{B}_i\coloneq \bigcup_{u \in \mathcal{N}_i} \mathcal{U}_i[u].$
A message contains the information whether it was sent via directed cast or broadcast. Agents can send and receive messages at any time.

Communication is situated, allowing agent $i$ to determine the node(s) from which a message reached the agent. 
An agent $i$ that receives a directed cast message from an agent $j\neq i$ via a shared, unoccupied, adjacent node $u$ can determine that this directed cast came from $u$. If $j$ is adjacent to $i$ and sends a directed cast message via the edge $\{v_i,v_j\} \in \mathcal{E}$ or $v_i=v_j$, $i$ knows that this message came from $v_j$.
An agent $i$ that receives a broadcast message that was sent by an agent $j \neq i$, gathers the nodes from which the broadcast reached $i$ in the set \texttt{NodesTowards}$_i(j)$ (adapted from~\cite{Jahir}).
Formally, 
\begin{equation*}
\texttt{NodesTowards}_i(j) =
\begin{cases}
\left\{v_j\right\}, 
& \text{if } v_j = v_i
\\
\left\{v_j\right\} \cup (\mathcal{N}^{\text{unoccupied}}_i \cap \mathcal{N}^{\text{unoccupied}}_j),
& \text{if } \{v_i,v_j\} \in \mathcal{E}
\\
\mathcal{N}^{\text{unoccupied}}_i \cap \mathcal{N}^{\text{unoccupied}}_j, 
& \text{otherwise}.
\end{cases}
\end{equation*}

All agents are tasked to reach the goal $g$. 
The following two evaluation criteria are used to measure the performance~\cite{SternSoCS19,Jahir}:
(i)  
\emph{Makespan}: The overall time it took for all agents to reach the goal, formally $\min\left\{k \in \mathbb{N}_0 \mid \forall a \in  \mathcal{A}: v_a[k] = g\right\}$;
(ii)  
\emph{Average sum-of-fuels}: The mean number of edges traversed to reach the goal, formally
    $1/n \sum_{i=1}^{n} \pi_i$, where $\pi_i$ is the number of edges traversed by agent $i$.

\begin{algorithm}[h!]
\DontPrintSemicolon
\caption{Multi-Agent Maze Traversal Algorithm}\label{mainalgorithm}

$v \gets s$, $v^{prev} \gets s$, $L \gets \texttt{nil}$\; \label{initbecomehead}\label{initat_v}\label{initat_v_old}
{\tcp{\textcolor{gray}{Synchronized sensing and messaging step}}}
Sense adjacent nodes\; \label{initsense}
Broadcast and receive \emph{status messages} from neighbor agents, Update $\mathcal{C}$\; \label{initcomms}  \label{algo:comagentsinit}
\uIf{\emph{status message} received from agent $l$}{ \label{initmessagereceived}
    $L \gets l$ \;\label{initsetleader}
}
\While{$v \ne g$}{
\label{rec1}
$D \gets v$, $competing \gets \texttt{false}$    \;\label{initD}\label{initcompeting}
{\tcp{\textcolor{gray}{Decision making step}}}
\uIf{$L = \text{nil}$}{ \label{starthead}
    {\tcp{\textcolor{gray}{Head decision step}}}
    \uIf{$g \in\mathcal{N}$}{ \label{headgoaladjacent}
    $D \gets g$ \; \label{headmovetogoal}
    }
    \Else{
        $D^{\text{solver}} \gets$ \texttt{SingleAgentMazeSolver}()\; \label{solver}
    \uIf{$D^{\text{solver}} \in \mathcal{N}^{occupied}$}{ \label{algo:Dsolveroccupied}
        $L \gets \texttt{Selector}(\{ a \in \mathcal{C} \mid D^{\text{solver}} \in \texttt{NodesTowards}(a)\})$ \;\label{algo:choosenewhead}
    }
    \ElseIf{$D^{\text{solver}} \neq D$}{\label{algo:keephead}
        Directed cast along $\{v, D^{\text{solver}}\}$ \; \label{algo:unicastalongdsolver}
        $D \gets D^{\text{solver}}$ \; \label{movetoDsolver}
    }
    }
    Broadcast \emph{head message}\;\label{PropagateDsolverHead}
} 

\Else{
\label{startnonhead}
        $D, L \gets \texttt{ResolveConflict}()$\; \label{checkconflict}

}\label{endnonhead}

    {\tcp{\textcolor{gray}{Synchronized movement step}}}
    \textbf{move to} $D$\; \label{move}
    $v^{prev} \gets v$, $v \gets D$ \; \label{updateVVprev}
    {\tcp{\textcolor{gray}{Synchronized sensing and messaging step}}}
    Sense adjacent nodes\; \label{sensealg1}
    \uIf{$v^{prev} \neq v$}{
    Directed cast along $\{v, v^{prev}\}$\; \label{unicastsendinfo}
    }
    Broadcast \emph{status messages} \; \label{broadcastsendinfo}
    Receive directed casts (optional) and \emph{status messages} \;
    Update $v^{prev}_L$, $\mathcal{C}$ \; \label{updateC} 
} 

\end{algorithm} 

\begin{algorithm}[h!]
\DontPrintSemicolon
\caption{ResolveConflict (as performed by agent $i$)}\label{alg:ResolveConflict}

\KwOut{The node to move to $D$ and the updated leader $L$}

\uIf{$\{ h \in \mathcal{C} \mid L_h = \textnormal{\texttt{nil}} \land v_h \neq g\} \neq \emptyset$}{ 
    \label{algo2:headincommunicationrange}
    $h \gets \texttt{Selector}(\{ h \in \mathcal{C} \mid L_h = \texttt{nil} \land v_h \neq g\})$ \;\label{algo2:extracthead} 
    {\tcp{\textcolor{gray}{Wait for head decision step to finish}}}
    Wait until directed cast (optional) and \emph{head message} from $h$ received \; \label{algo2:waitforheadmessage}
\uIf{\emph{head message} contains $L_{h} = i$}{\label{algo2:headtransferchosen}
    
    Broadcast \emph{competing message}  \;\label{algo2:headtransferbroadcastcompeting}
    \Return{$v, \textnormal{\texttt{nil}}$}\label{algo2:headtransferreturnvnil}
}

\uElseIf{$\text{Directed cast from $h$ received} \lor 
(
\textnormal{\texttt{NodesTowards}}(h)\cap \mathcal{N}^{\text{occupied}} \neq \emptyset)$}
{\label{algo2:headadjacentorunicast}
   
    Broadcast \emph{competing message} \;
    \Return{$v$, $h$} \;   \label{algo2:returnheadadjacentorunicast}
}
}
\uIf{$g \in\mathcal{N}$}{ \label{algo2:goaladjacent}
    
    Broadcast \emph{competing message} \;\label{algo2:goaladjacentbroadcastcompeting} 
        \Return{$g$, $L$} \label{algo2:goaladjacentreturn}
    }

    \uIf{\textnormal{\texttt{NodesTowards}}$(L) \cap \mathcal{N}^{occupied} \neq \emptyset$}{ \label{algo2:leaderadjacent}
    
    Broadcast \emph{competing message}  \;\label{algo2:leaderadjacentbroadcastcompeting}
        \Return{$v$, $L$} \label{algo2:leaderadjacentreturn}
    }

    $competing \gets \texttt{true}$ {\tcp*{\textcolor{gray}{Compete for $v^{prev}_L$}}} \label{algo2:competingin}
    Broadcast \emph{competing message} \; \label{algo2:broadcastcompeting}
    Wait until \emph{competing messages} from all agents $a \in \mathcal{C}$ received\; \label{algo2:waitforcompeting}

    $\mathcal{R} \gets \{ a \in \mathcal{C} \mid v_a \ne g \land competing_a \land L_a \in \mathcal{C} \land v^{prev}_{L_a} = v^{prev}_L\}$  \label{alg2:getR} \; 
    
    \uIf{
    \textnormal{\texttt{Selector}}$(\text{$\mathcal{R} \cup \{i\}$}) = i$} { \label{alg2:conditiontokeep}
        \Return $v^{prev}_L$, $L${\tcp*{\textcolor{gray}{Keep leader}}}\label{alg2:keepL}
    }
    \Else{
        \Return $v$, $\texttt{Selector} (\text{$\mathcal{R} \cup \{i\}$})${\tcp*{\textcolor{gray}{Update leader}}}\label{alg2:updateL}
    }

\end{algorithm}

\section{Algorithm Design}
\label{sec:algorithm}
We extend the approach in~\cite{Jahir} to cyclic graphs while preserving its core algorithmic idea. In each time step (until reaching a node adjacent to the goal), one active agent (referred to as the \emph{head}) employs a single-agent maze solver to explore the maze.
All other active agents choose an active agent, their \emph{leader}, to follow such that the resulting leader-follower-graph is a tree; hence, they directly or indirectly follow the head (Fig.~\ref{fig:mazelead}).
All agents are assumed to possess any capabilities that are needed for the underlying single-agent maze solver in addition to the requirements discussed in Section~\ref{subsec:complexity}.

When generalizing this approach to cyclic mazes, additional considerations are required as multiple paths may exist between a pair of nodes.
The proposed algorithm lets agents send a \emph{directed cast} along a single, chosen edge, modeling directional communication along individual corridors. After an agent $i$ moves from any node $v$ to new node $u$, it sends a directed cast along edge $\{u,v\}$. 
This informs its \emph{followers} (i.e., active agents $j$ that have agent $i$ as their leader) about agent $i$'s previous node $v$.
Additionally, the head can send a directed cast to inform any agents in its communication range of its next move.

Similarly to~\cite{Jahir}, the agents can employ broadcast messages. In the present work, three types of broadcast messages are used:
\begin{enumerate}
    \item The \emph{status message} is sent once per time step by each agent $i \in \mathcal{A}^{\text{active}}$ and contains its ID, whether $i$ is at the goal, (i.e., $v_i = g$) and $i$'s leader pointer $L_i \in (\mathcal{A}^{\text{active}}\setminus \{i\}) \cup \{\texttt{nil}\}$.
    \item The \emph{competing message} is sent once per time step by all active non-head agents and contains only their ID and the flag \emph{competing}, which indicates if they intend to move and should therefore be considered competing in potential conflicts for nodes in this time step.
    \item The \emph{head message} is sent once per time step by the current head agent, $h$, containing its ID, the flag \emph{competing}, the state of the single-agent maze solver and information on the head agent at the next time step $L_h$ (either $\texttt{nil}$ if the current head keeps the role, and the ID of another agent otherwise).
\end{enumerate}

The proposed algorithm is shown in Algorithm~\ref{mainalgorithm}, seen from the point of view of the executing agent $i$.  It is executed by all agents in a synchronous, discrete-time manner.
Subscripts $i$ are omitted for better readability.

\emph{Initialization:}
(Line 1) When $i$ first becomes active, it is at node $s$ and has no leader.
(Line 2) Agent $i$ then senses the adjacent node (as $s$ is a leaf) to identify the neighborhood as well as if the adjacent node is the goal or if it is occupied.
(Line 3) If agent $i$ is the first agent activated, it will not receive any messages, resulting in $\mathcal{C} = \emptyset$. If other agents have been activated before $i$, then there exists another active agent $l$ that just left the start node and moved to the single adjacent node. In this case, agent $i$ will receive a status message from $l$, and $\mathcal{C} = {l}$.
(Lines 4--5) If $\mathcal{C} = {l}$, agent $i$ will make $l$ its leader, $L_i=l$. Otherwise, it keeps the leader pointer $L_i = \texttt{nil}$ and assumes the head role.

\emph{While Loop:}
The agent executes the while loop (lines 6--28) until it reaches the goal $g$.
Each iteration of the loop takes one time step.
(Line 7) At the beginning of each time step, $i$ resets its target node $D$ to its current node $v$ and competing to $\texttt{false}$.
\emph{While Loop Head Decisions:}
(Line 8) If $L_i = \texttt{nil}$, then agent $i$ is the head agent.
(Lines 9--10) If the goal is adjacent, $i$ sets its target node $D = g$.
(Line 12) Else, agent $i$ runs one step of the single-agent maze solver to determine an adjacent potential target node $D^{\text{solver}}$.
(Lines 13--14) If $D^{\text{solver}}$ is currently occupied by an agent $j$, $j$ is chosen as the new head and leader of $i$.
(Messages cannot pass through occupied nodes, hence $D^{\text{solver}} \in \texttt{NodeTowards}(a)$ will only be fulfilled for agent $j$, and the \texttt{Selector}\footnote{The \texttt{Selector} operator is any deterministic operator that, given a set of agents $A \subseteq \mathcal{A}$, returns the agent $a \in A$ with $a \leq_\mathcal{A} b$ for all $b \in A\setminus\{a\}$.} operator returns $j$.)
Else, $D^{\text{solver}}$ is either unoccupied or $i$'s current node, and $i$ retains the head role.
(Line 17) If $D^{\text{solver}}$ is unoccupied, agent $i$ sends a directed cast along edge ${v, D^{\text{solver}}}$ and updates its target node $D$ accordingly.
(Line 18) Then, $i$ broadcasts a head message informing all other agents within its communication range whether one of them has been elected as the new head.
\emph{While Loop Follower Decisions:}
(Line 20) If at time $k$ agent $i$ has a leader $L_i$, it executes routine \texttt{ResolveConflict} (discussed below), which is shown in Algorithm~\ref{alg:ResolveConflict}.
\texttt{ResolveConflict} returns a target node $D$ and the potentially updated leader $L_i$.

\emph{Movement:}
(Lines 21-22) After completing its decision step, agent $i$ moves synchronously with all other agents to its target node. It updates $v^{prev}$ and $v$.

\emph{Post-Movement Sensing and Messaging:}
(Lines 23--26) Agent $i$ senses adjacent nodes. If $i$ moved to a new node, it broadcasts a directed cast along $\{v, v^{prev}\}$, and always broadcasts a status message.
(Lines 27--28) Agent $i$ receives directed casts and status messages from nearby agents and updates $v^{prev}_L$ and $\mathcal{C}$ accordingly.
Agents who terminated after reaching the goal in the previous time step will not send a message, and are removed from $\mathcal{C}$.
Lines 23--28 are synchronized with lines 2--5 of a newly activated agent.

\emph{ResolveConflict:}
(Line 1) Agent $i$ checks whether the head $h$ 
is in communication range and not yet at the goal.
(Lines 2–3) If the head $h$ is in communication range, agent $i$ waits for the head message.
(Lines 4--6) Agent $i$ checks if it has been chosen to be the new head. If so, it broadcasts that it is not competing, sets its leader $L_i = \texttt{nil}$, and sets its target to its current node.
(Lines 7--9) If agent $i$ was not chosen as the new head, it checks if the agent $h$ is currently adjacent or will be adjacent in the next step. In this case, it sets the head as its leader, does not compete to move, and sets its target to its current node.
If agent $i$ was not chosen as the new head, it checks if it received the directed cast from $h$ or if $h$ is currently in an adjacent node. If either is true, $i$ broadcasts that it is not competing, sets its leader $L_i = h$, and sets its target to its current node.
(Lines 10--12) If the current head is not in communication range or neither of the above cases apply, agent $i$ checks if the goal is adjacent and if so, does not compete for other nodes, maintains its current leader $L$, and sets its target to the goal.
(Lines 13--15) Else, if its (non-head) leader is residing in an adjacent node, $i$ keeps this leader and does not move.
(Lines 16--18) If none of the above cases apply, agent $i$ competes to move to the node previously occupied by its leader. It broadcasts $competing = \texttt{true}$ and waits until it receives information on which agents in communication range are competing.
(Line 19) Agent $i$ determines which other agents are competing. An agent a is considered competing if it is not yet at the goal node, has signaled that $competing = \texttt{true}$, and wants to move to the same node as agent $i$.
(Lines 20--23) The \texttt{Selector} function is called on the set of all competing agents and returns a winner. If agent $i$ is the winner, it keeps its previous leader and selects as target the previous node of that leader. Else, $i$ chooses the winner as its new leader and remains at its current node.

\begin{figure}[t]
    \centering
    \subfloat[$k=0$]{
        \includegraphics[width=.15\linewidth]{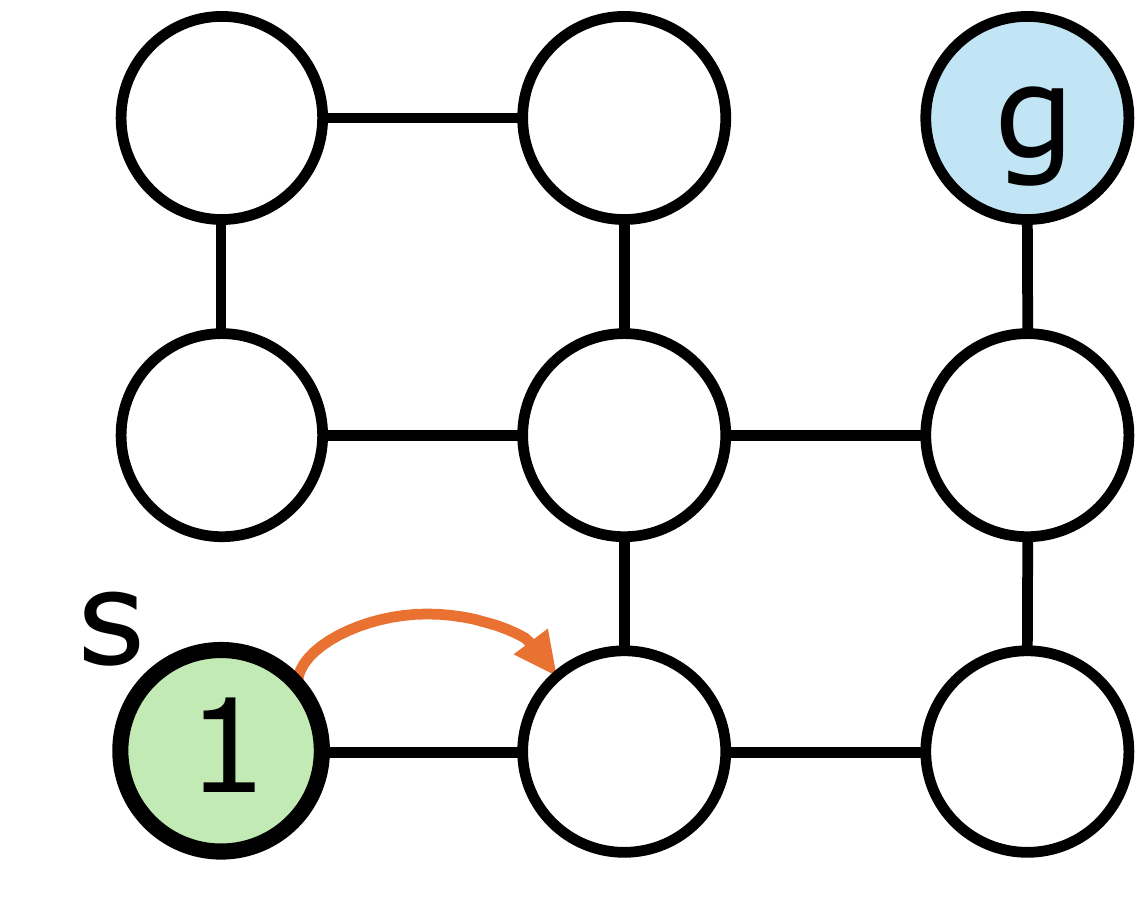}
    }
    \hfill
    \subfloat[$k=1$]{
        \includegraphics[width=.15\linewidth]{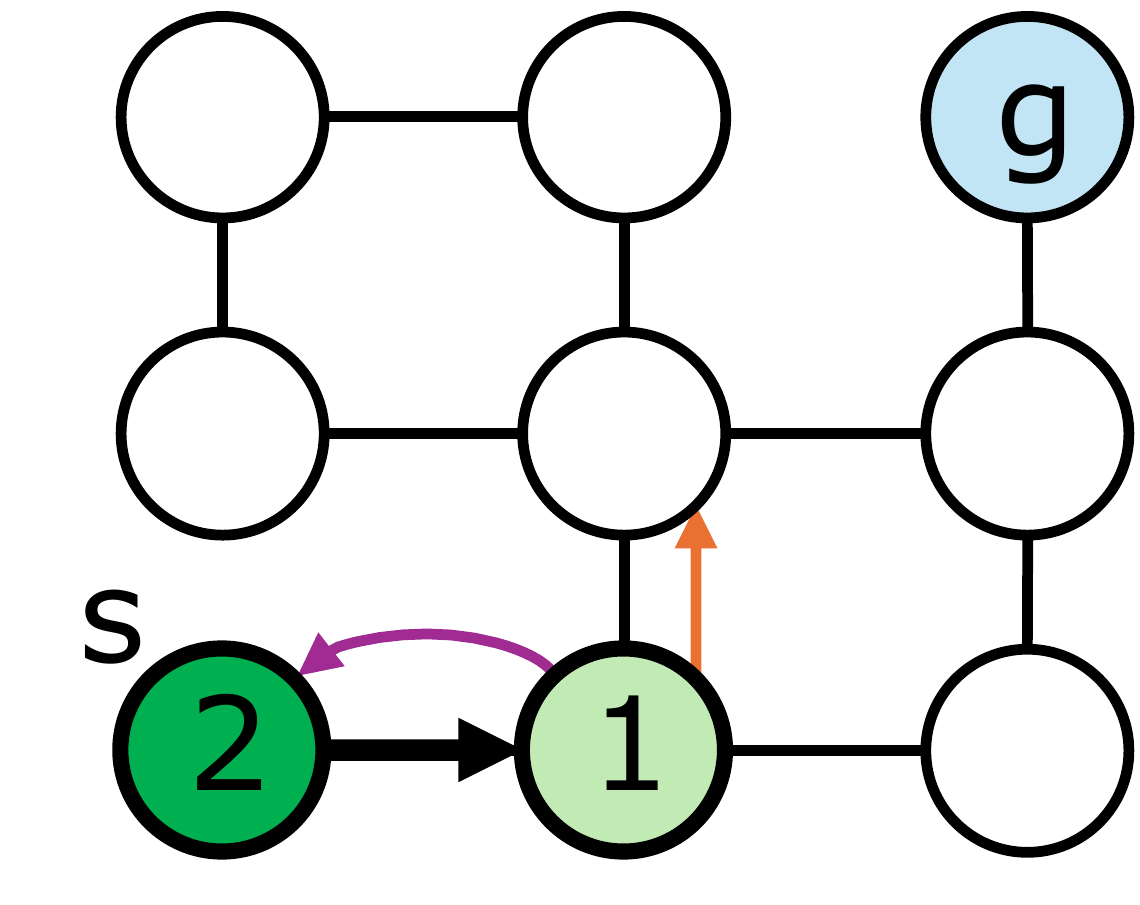}
    }
    \hfill
    \subfloat[$k=2$]{
        \includegraphics[width=.15\linewidth]{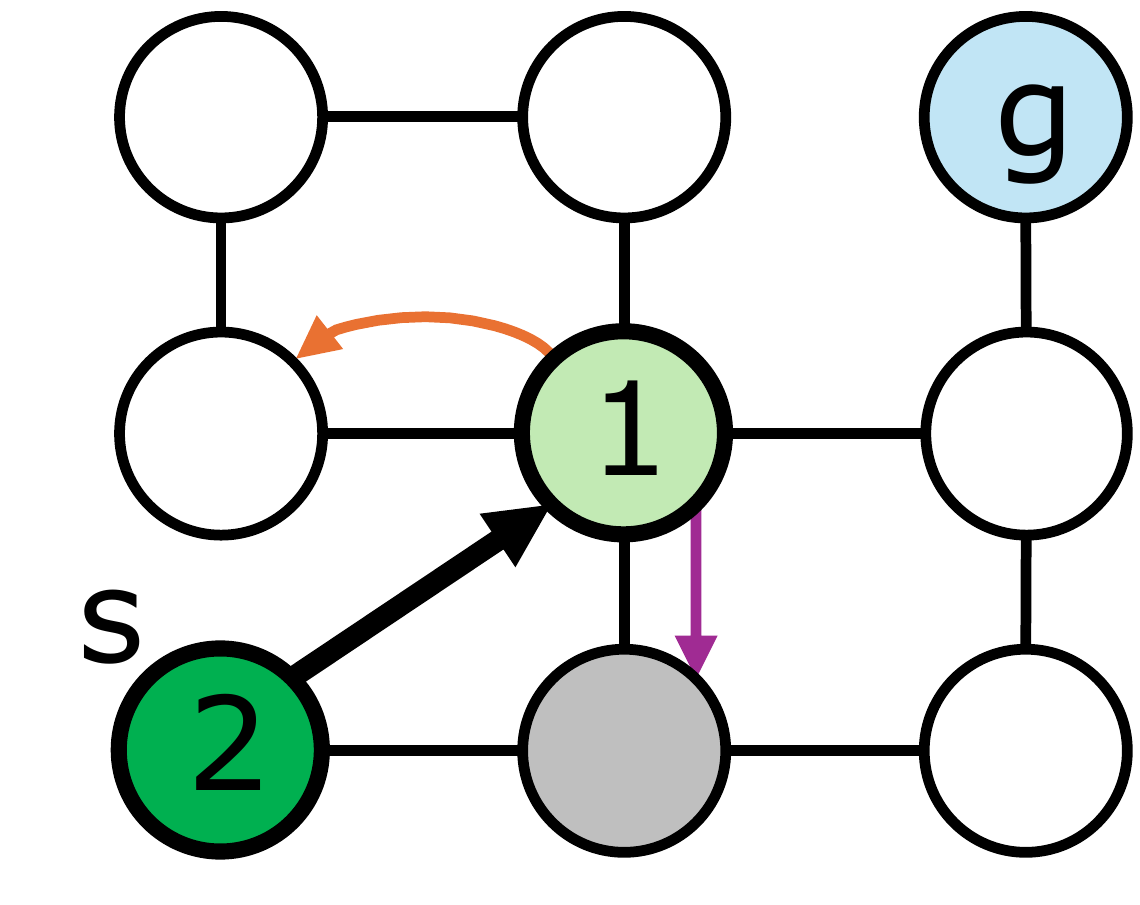}
    }
    \hfill
    \subfloat[$k=4$]{
        \includegraphics[width=.15\linewidth]{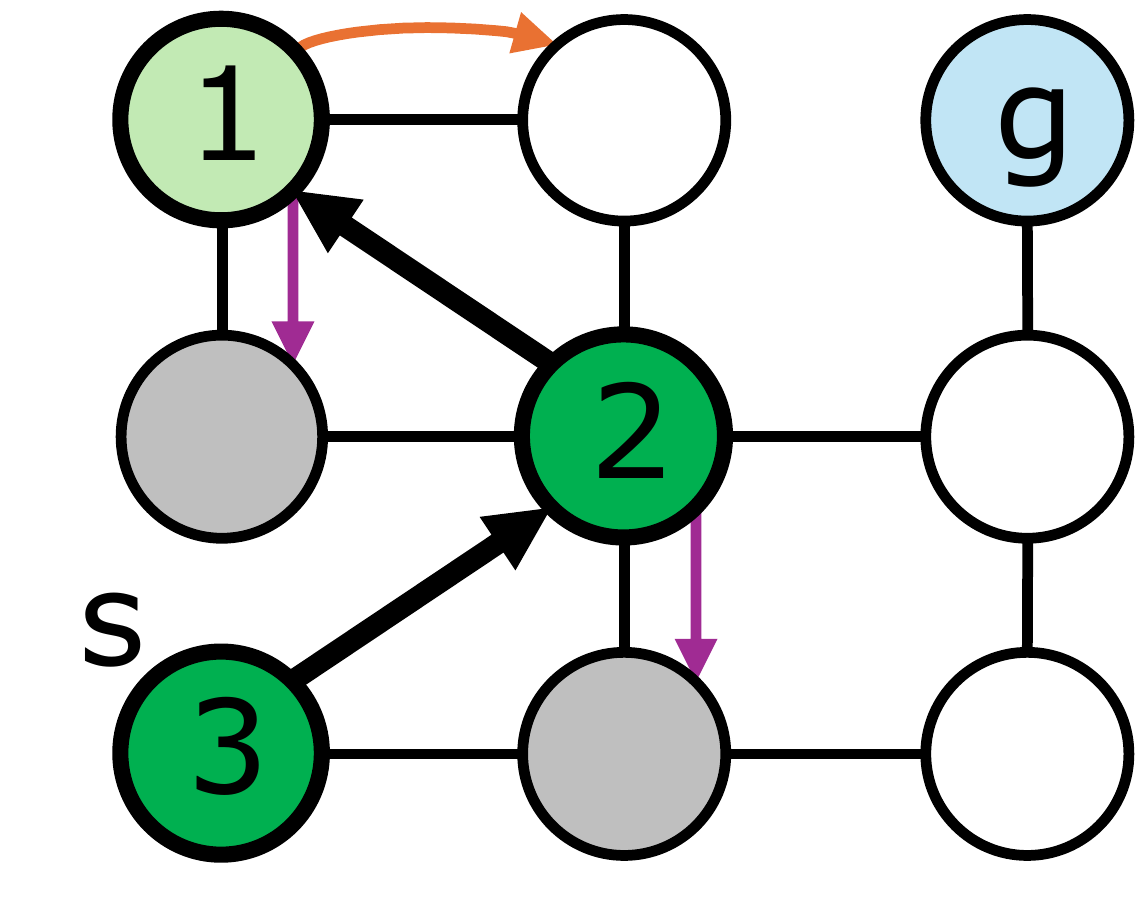}
    }
    \hfill
    \\
    \subfloat[$k=5$]{
        \includegraphics[width=.15\linewidth]{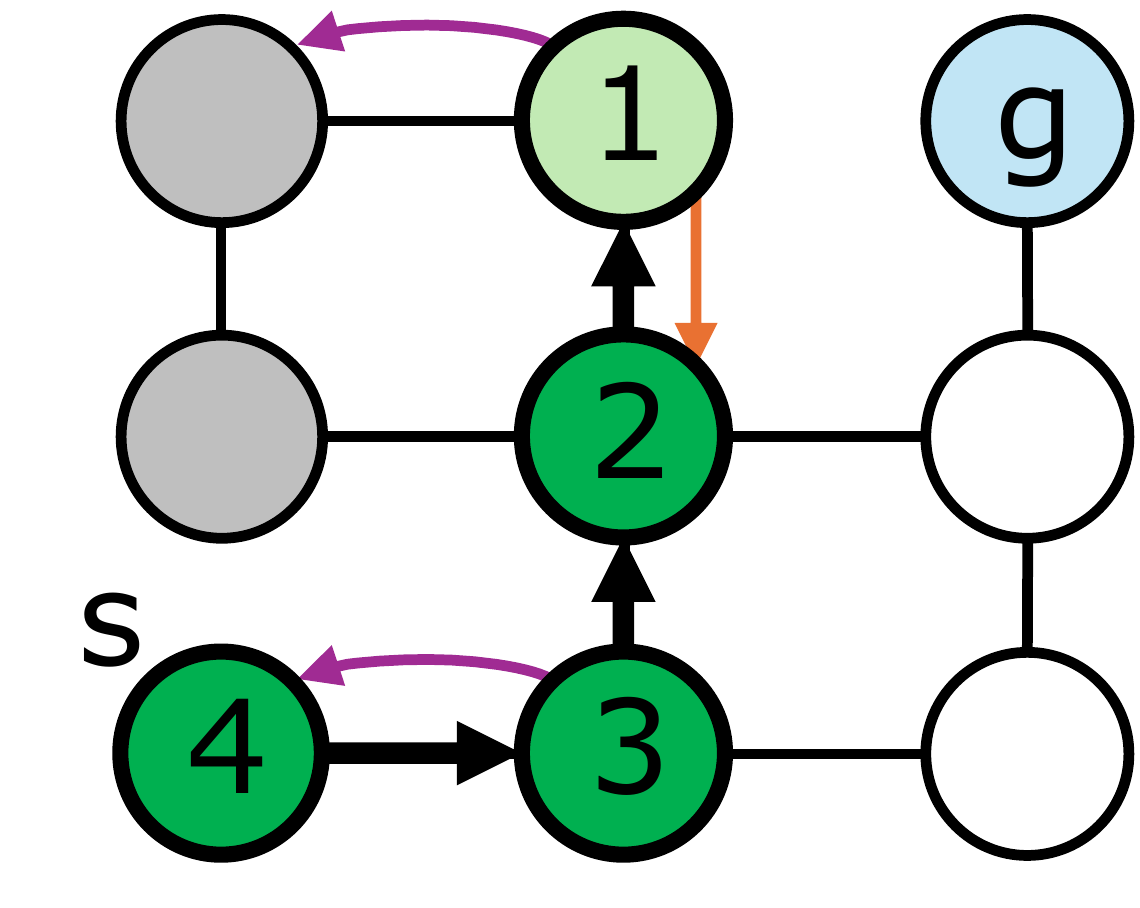}
    }
    \hfill
    \subfloat[$k=6$]{
        \includegraphics[width=.15\linewidth]{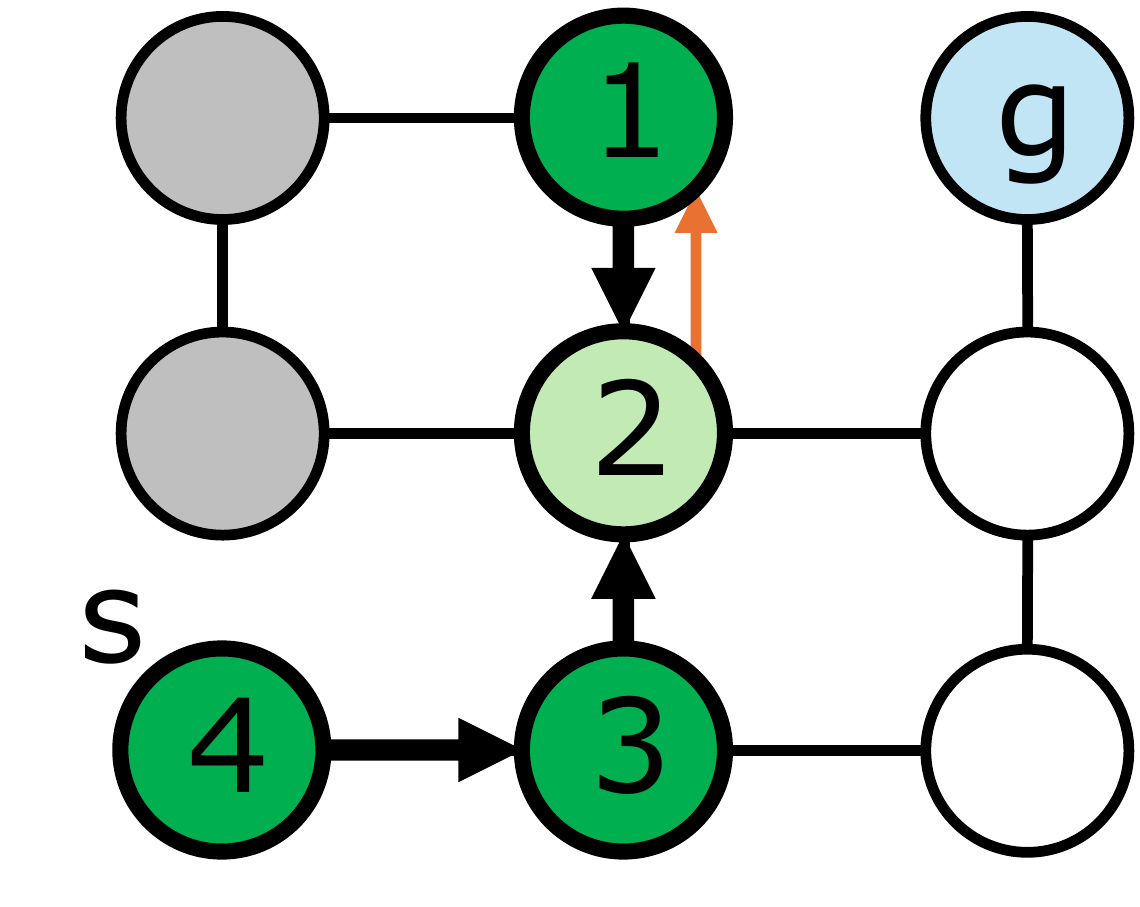}
    }    
    \hfill
    \subfloat[$k=11$]{
        \includegraphics[width=.15\linewidth]{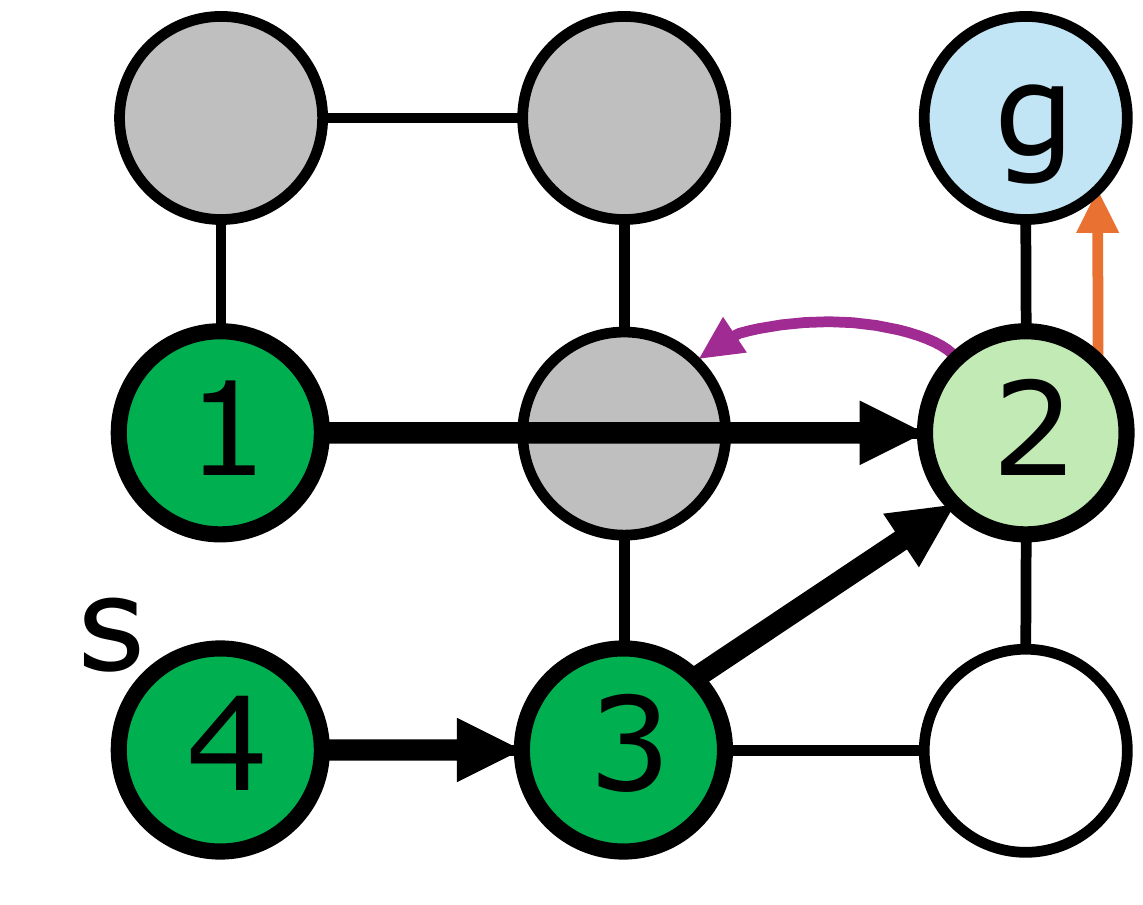}
    }
    \hfill
    \subfloat[$k=12$]{
        \includegraphics[width=.15\linewidth]{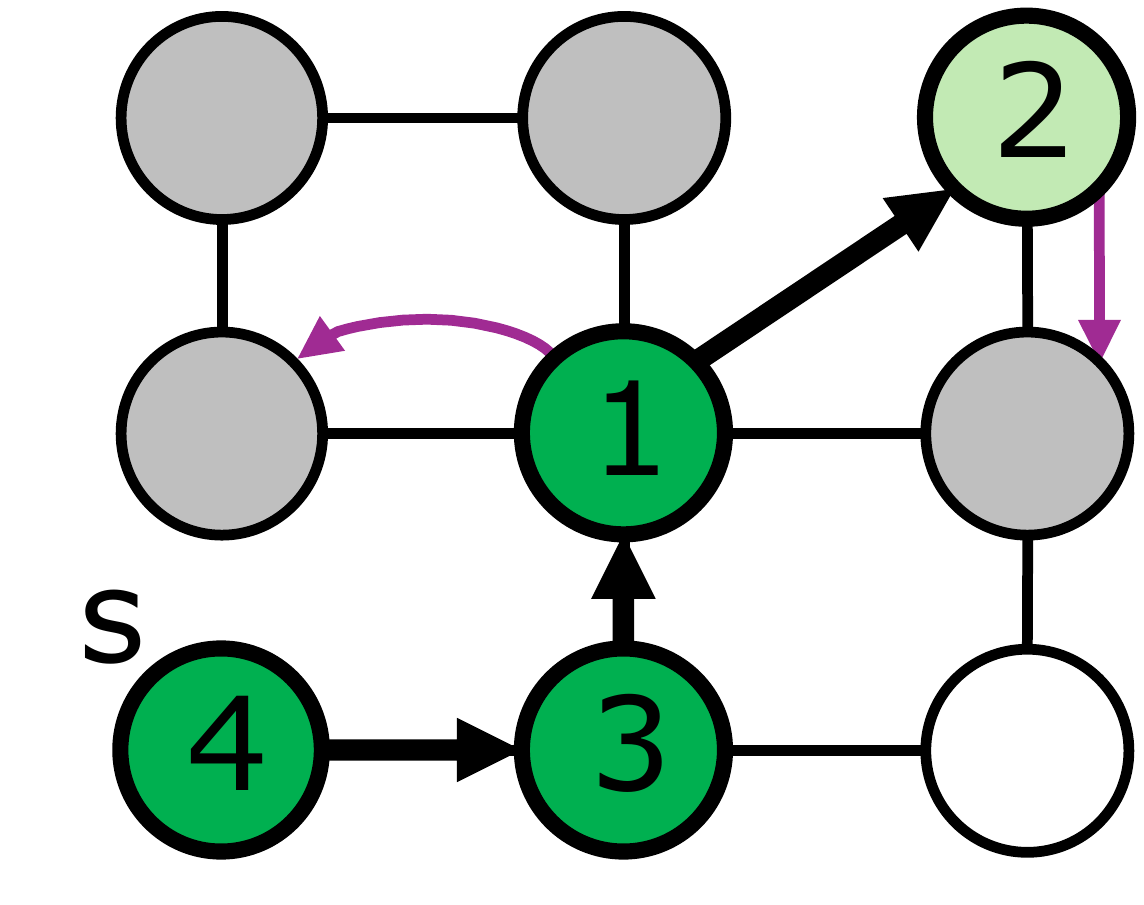}
    }    
    \caption{An example of our algorithm (employing Tr\'emaux's algorithm). Gray nodes have been visited. $k=0$: All agents are at start node $s$; agent $1$ is active, claiming the head role (light green). $k=1$: Agent $1$ leaves $s$ (now its previous node, purple arrow) following Tr\'emaux's algorithm (orange arrow). Agent 2 becomes active (dark green) and chooses agent $1$ as leader (black arrow). $k=2\text{--}4$: Agents $1$ and $2$ gradually navigate the maze, agent $3$ becomes active. $k=5$: Agent $2$ receives the head's directed cast and remains in its current node. Agent $1$ chooses agent $2$ to become the new head at $k=6$. $k=7\text{--}10$: Agent $2$ transfers the head back to agent $1$ which performs backtracking. $k=11$: Agent $2$ is again the head and moves to a node adjacent to the goal. $k=12$: Agent $2$ reaches the goal. Agents $1$ and $3$ compete for agent $2$'s previous node; Agent $1$ wins and moves to the node. Agent $3$ chooses agent $1$ as its new leader.}
    \label{fig:example}
\end{figure}

Fig.~\ref{fig:example} provides an example that highlights particular moments of a group of agents executing the aforementioned algorithms.

\subsection{Mathematical Analysis}
We now prove the correctness of the algorithm and several properties.

\begin{lemma}
\label{lemma:onlyonehead}
    At all discrete times $k \in \mathbb{N}_0$, there is exactly one agent assuming the head role. Let $k_w \in \mathbb{N}_0 \cup \{\infty\}$ be the time at which a single agent running a single-agent maze solver resides in a node adjacent to the goal for the first time. At all time steps $k \leq k_w$, using the same single-agent maze solver, the head agent occupies the same node as the single agent would occupy at time $k$. At all times $k > k_w$, the head agent resides at the goal node.
\end{lemma}
\begin{proof}
    Proof by induction: $k=0$: At the beginning of the first time step, there is only one active agent $h$. $h$ resides in the start node $s \neq g$ as would an agent executing the single-agent maze solver and is the only head by default. 
    
    $k \to k+1$: 
    An agent $a$ that activates in time step $k+1$ will 
    activate only because another agent $b$ just left the start node and moved to an adjacent node. Consequently, $a$ is in communication range of $b$ and will receive a status message from $b$, choosing $b$ as its leader. At time $k+1$, $a$ is not the head.
    
    Per the induction assumption, there is exactly one head $h$ at time $k$ that is at the same node as a single agent executing the single-agent maze solver would be.  We distinguish between three cases.
    
    Case 1: If $k < k_w$, $h$ will use the single-agent maze solver to determine where to move to.  
    If the determined adjacent node $D^{\text{solver}}$ is not occupied, $h$ will keep $L_h = \texttt{nil}$, set $D = D^{\text{solver}}$ and inform agents in communication range (lines~\ref{algo:keephead}--\ref{PropagateDsolverHead}). Any agent that has $h$ in communication range will receive the head message (line~\ref{algo2:waitforheadmessage} of Algorithm~\ref{alg:ResolveConflict}). As $h$ sent $L_h = \texttt{nil}$, no agent $i$ will 
    become the head (lines~\ref{algo2:headtransferchosen}--\ref{algo2:headtransferreturnvnil}).
    Moreover, there is no other moment in time where an agent can set its leader to \texttt{nil}. As $h$ will move to $D^{\text{solver}}$ in time step $k+1$, at time $k+1$, there is exactly one head, agent $h$, which resides in the same node as a single agent executing the single-agent maze solver would.
    If $D^{\text{solver}}$ is occupied at time $k$, agent $h$ chooses the occupying agent $L$ as its new leader and sends $L_h = L$ to all agents in communication range via the head message. As $L$ is occupying an adjacent node, it is in communication range of $h$ and will wait for and receive the head message. As $L_{h} = L$, $L$ sets its leader to \texttt{nil} and assumes the head role. As every agent has a unique $ID$, no other agent $j \in \mathcal{C}_h\setminus\{L\}$ will fulfill $L_{h} = j$ and assume the head role. Consequently, at time $k+1$, there is exactly one head, agent $L$, which resides in the same node as the single-agent maze solver would.
    
    Case 2: If $k = k_w$, then $h$ is currently in a node adjacent to the goal. Consequently, it will move to the goal in time step $k+1$. In doing so, it will remain the head and no other agent will receive a request to assume the head role. 
    
    Case 3: If $k > k_w$, then $h$ is at the goal node 
    and no longer sending messages, and will never transfer the head.
    Hence, no other agent will set its leader to $\texttt{nil}$ as $\{ h \in \mathcal{C} \mid L_h = \texttt{nil} \land v_h \neq g\} = \emptyset$. 
    In both latter cases, at $k+1$, there is exactly one head which resides at the goal.
\end{proof}

\begin{lemma}
\label{lemma:traverseviasamenode}
    The first visit to any node $v \in \mathcal{V} \setminus \{ s \}$ is done by the current head agent. Any agent that reaches the goal will do so via the same node adjacent to the goal that was traversed by the head.
\end{lemma}
\begin{proof}
In the case where $s$ and $g$ are adjacent, at $k=0$, there is only one active agent. This agent assumes the head role, moves to $g$ (where it will eventually terminate). As all other agents activates afterwards, they also sense the goal being adjacent, move there and terminate. Consequently, the only node $\neq s$ that is visited is $g$ and it is first visited by the head. All agents move to $g$ via $s$.

Assume that $s$ and $g$ are not adjacent. 
Assume that there exists a node $w \in \mathcal{V}\setminus\{s,g\}$ that has not been visited before and is now visited by a non-head agent $a$.
$a$ has only two possible options to move to another node (lines~\ref{algo2:goaladjacentreturn} and \ref{alg2:keepL} of Algorithm~\ref{alg:ResolveConflict}). Since $w \neq g$, $a$ must have moved according to line~\ref{alg2:keepL}, that is, to a node that was previously visited by its leader $L$, which is a contradiction to the fact that $a$ is the first agent to visit $w$. 
Let $u \in \mathcal{V}\setminus\{s,g\}$ be the first node adjacent to the goal that the head agent visited. As the goal is adjacent, the head moves there from $u$ and terminates. As shown above, no other node adjacent to the goal will be visited by any agent, as the head has already terminated at the goal. Thus, no other agent reaches a node adjacent to the goal before the head agent, and the goal node 
$g$ is first visited by the head agent. Moreover, as $u$ is the only adjacent node to the goal that is visited by any agent, it directly follows that every agent that reaches the goal will do so by traversing via $u$.
\end{proof}

\begin{lemma}
\label{lemma:nocollisions}
    For any node $v \in \mathcal{V}\setminus\{g\}$, at all discrete times $k$, $v$ is occupied by at most one active agent. If an agent leaves node $v$ in time step $k$, no other agent that was already active at $k-1$ will move to $v$ in the same time step. 
    In time step $k$, all agents competing for the same node $v$ have the same leader.
\end{lemma}
\begin{proof}
    Proof by strong induction: At time $k=0$, all agents are at $s$, all but one are inactive and the lemma holds for any $v \in \mathcal{V}\setminus\{g\}$.
    Let $k_0 \in \mathbb{N}$ be the first time step where the head leaves the start node and hence occupies an adjacent node. As soon as the head leaves the start node, one agent $a$ that is not the head becomes active at the start node. Again, the lemma is true.
    
    $1, 2, \dots, k \to k+1$: Let $v \in \mathcal{V}\setminus\{g\}$. Per induction assumption at $k$, there is no more than one active agent at $v$. There are only two options how agents can move to $v$: either as the (only) head (Lemma~\ref{lemma:onlyonehead}) moving to $D^{\text{solver}} = v$ or following their leader to $v^{prev}_L = v$. In the first case, if $v$ is currently occupied, the head $h$ chooses the occupying (active) agent (which is only one due to the induction assumption) as new head and does not move (lines~\ref{algo:Dsolveroccupied}--\ref{algo:choosenewhead}). Hence, $h$ will not move to a node in time step $k+1$ that is occupied at time $k$. If $v$ is not occupied, $h$ sends out a directed cast along $D^{\text{solver}} = v$ which all agents in $\mathcal{U}_h(v)$ will receive (i.e., all other agents adjacent to $v$). They choose $h$ as their leader and remain in their current node. Hence, $h$ can safely move to $v$ without collision, that is, there is only one active agent at $v$ at time $k+1$.
    
    In the second case, an agent $a$ aims to move to the previous node $v^{prev}_{L_a} = v$ of its leader $L_a$. It can only do so if it received the directed cast from its leader $L_a$ in the previous time step, that is, $a \in \mathcal{U}_{L_a}(v)$ and $L_a \in \mathcal{C}_a$, if it is selected by the $\texttt{Selector}$-operator, and if the head does not aim to move there in this time step. If $v^{prev}_{L_a}$
    is occupied, $L_a$ must be residing there and $a$ does not move (lines~\ref{algo2:leaderadjacent}--\ref{algo2:leaderadjacentreturn}) as else, $a$ would not have received the directed cast from $L_a$. If $v^{prev}_{L_a}$ is unoccupied, $a$ receives all competing messages from agents in communication range. Every non-head agent $b$ that could move to $v^{prev}_{L_a}$ fulfills the following: 
    (i) $b \in \mathcal{C}_a$, as $b$ must be adjacent to $v^{prev}_{L_a}$ to be able to move there, so $a$ and $b$ are two hops away with an empty node in between, 
    (ii) $v_b \neq g$ as else, $b$ would not move,
    (iii) $competing_b = \texttt{true}$, as this is only sent if $b$ not already decided to remain in its current node nor move to $g$, (iv) $L_b \in \mathcal{C}_a$ and $v^{prev}_{L_b} = v^{prev}_{L_a}$, as else, $b$ would not aim to move to $v^{prev}_{L_a}$. 
    Per induction assumption at $k-1$, there was at most one agent residing at $v^{prev}_{L_a}$, which means that in time step $k+1$, $v^{prev}_{L_a}$ can only be the previous node of $L_a$ and no other agent. Hence, every non-head agent that aims to move to and therefore competes for $v^{prev}_{L_a}$ in time step $k+1$ does so because $L_a$ is its leader.
    Consequently, $a$ has every competing agent in $\mathcal{R}_a$ (line~\ref{alg2:getR}). It directly follows that $b \in \mathcal{R}_a \Longleftrightarrow a \in \mathcal{R}_b$ and $\mathcal{R}_a \cup \{a\} = \mathcal{R}_b \cup \{b\}$. All agents competing for $v$ choose the same agent as winner by applying the (deterministic) $\texttt{Selector}$-operator on the same set, resulting in only one agent that is allowed to move to $v$. 
    
    Hence, at time $k+1$, there is at most one active agent at $v$. As agents only move simultaneously and to currently unoccupied nodes, no other agent that is active at $k$ will aim to move to $v$ in time step $k+1$ if $v$ is occupied at $k$. 
\end{proof}

\begin{corollary}
    Under the algorithm, vertex and following conflicts cannot occur.
\end{corollary}

\begin{lemma}
\label{lemma:leaderincommunicationrange}
    At all discrete times $k$, if an active agent $a$ is not the head, it has its current leader $L_a$ in communication range and one or more of the following three cases are true:
    \begin{enumerate}
        \item the previous node of its leader $v^{prev}_{L_a}[k]$ is unoccupied and adjacent to $v_a[k]$ and $v_{L_a}[k]$
        \item the current node of $a$'s leader $v_{L_a}[k]$ is adjacent to $a$'s current node $v_a[k]$
        \item $v_a[k]=v_{L_a}[k]=g$.
    \end{enumerate}
\end{lemma}

\begin{proof}
    See Appendix \ref{sec:proofs} for the complete proof. 
    Non-head agents move to follow their leader. Agents only compete to move when their leader has moved to a non-adjacent node. An activated agent starts with their leader in an adjacent node by default, so this occurs only when a leader moves two nodes away, leaving an unoccupied node. The leader is still in communication range at this point, but may move again. Agents following this leader now compete for the unoccupied space. If an agent wins this competition, it moves to follow its leader. Otherwise, another agent moves to this space and is chosen as new leader, maintaining proximity and communication network cohesiveness. Leaders may also stop communicating when they reach a goal. In this case, agents still compete for the previously occupied node, which brings them into proximity of the goal and maintains the network connectivity for remaining agents.
\end{proof}

\begin{lemma}
\label{lemma:leadershipistree}
    At all discrete times $k$, an active agent follows the head directly or indirectly. Formally, at discrete time $k$, let the directed leadership graph be $\mathcal{T}_L[k] = (\mathcal{A}^{\text{active}}[k], \mathcal{E}_L[k])$, $a,b \in \mathcal{A}^{\text{active}}[k]: b \text{ is leader of } a \Longleftrightarrow (a,b) \in \mathcal{E}_L[k]$. Then, $\mathcal{T}_L[k]$ is an \emph{in-tree}~\cite{mehlhorn2008algorithms} with the current head agent $h$ as its root, that is, for all $a \in \mathcal{A}^{\text{active}}[k]\setminus\{h\}$, there exists exactly one path to $h$ in $\mathcal{T}_L[k]$.
\end{lemma}

\begin{proof}
    See Appendix \ref{sec:proofs} for the complete proof. 
    For an agent to not be following the head (either directly or indirectly), it would have to be following another agent for which no path exists leading to the head. This would require either a cycle within the tree created by following agents, or for two agents to be acting as the head. Lemma~\ref{lemma:onlyonehead} shows that only one head exists at a time. A cycle in leadership can only arise when the head passes the head role to a new agent who is choosing to follow the current head. However, once an agent receives a head message informing it of its role as new head, it removes the following edge to the former head. Additionally, when agents win competitions for moving to unoccupied nodes, they maintain their leader and losing agents select the winner as their new leader, maintaining the in-tree behavior of leadership.
\end{proof}

\begin{theorem}
\label{theorem:correctness}
    If a single agent executing a single-agent maze solver reaches a node adjacent to the goal for the first time in time step $k_g - 1$, then $n$ agents executing Algorithm~\ref{mainalgorithm} (using the same single-agent maze solver) reach the goal in at most $k_g + 2 (n-1)$ time steps.
\end{theorem}
\begin{proof}
    If the goal node is adjacent to the start node, the single agent is already in an adjacent node at time $k=0$, that is, $k_g = 1$. Executing Algorithm~\ref{mainalgorithm}, the first agent moves to $g$ in the first time step (lines~\ref{headgoaladjacent}--\ref{headmovetogoal}), that is, it will be there at $k=1$. The newly activated agent moves to the goal in the next time step (Algorithm~\ref{alg:ResolveConflict}, lines~\ref{algo2:goaladjacent}--\ref{algo2:goaladjacentreturn}), and this repeats for all other agents, that is, all agents will reach the goal at time $n \leq k_g + 2 (n-1)$.

    Now, let $s$ and $g$ be not adjacent. Let $w \neq s$ be the first node adjacent to the goal visited by the single agent in time step $k_g - 1$.
    Per Lemma~\ref{lemma:onlyonehead}, the current head agent $h$ is in node $w$ at time $k_g-1$, moves to the goal node in the next time step, $v_h[k_g]=g$, and send out a final directed cast along $w$ and status message (including that it is now at the goal node) before it terminates. This means that no other agent receives the head role after $k_g$.

    Let $k_a$ be the time step in which an agent $a$ moves to the goal node. Let $n_a - 1$ be the number of agents that are not yet at the goal node at time $k_a$. 
    We prove the following statement via induction: 
    \emph{If an agent $a$ reaches the goal at $k_a$, then all $n_a \leq n$ agents reach the goal in at most $k_a + 2(n_a - 1)$ time steps.}

    $n_a=1$: If $a$ reaches the goal at $k_a$, there are $n_a-1=0$ agents left on nodes other than the goal. All agents reached the goal at $k_a=k_a+2(n_a-1)$.
    $n_a \to n_a + 1$: If $a$ moved to the goal at $k_a$, there are $(n_a + 1) - 1 = n_a > 0$ agents not at $g$. None of these agents can assume the head role (Lemmas~\ref{lemma:onlyonehead} and \ref{lemma:traverseviasamenode}).
    Assume that there is an agent $b$ of these $n_a$ agents that is not a direct or indirect follower of $a$. As $b$ is a direct or indirect follower of the head (Lemma~\ref{lemma:leadershipistree}), there exists exactly one path in the directed leadership graph $\mathcal{T}_L[k_a]$ from $b$ to the head $h$. Let the set $\mathcal{B}$ include all agents that are between $b$ and $h$ on this path (especially, $a \notin \mathcal{B}$). If $\mathcal{B}=\emptyset$, $b$ is a direct follower of $h$. Let $c \in \mathcal{B} \cup \{b\}$ be the agent on this path that is the direct follower of $h$. $h$ has already been at the goal node for at least two time steps (i.e., $v_{prev_h} = g$), as it must have moved there before agent $a$ moved to $w$ at time $k_a-1$. Hence, $c$ must also be at the goal node (Lemma~\ref{lemma:leaderincommunicationrange}), as it cannot be adjacent to $g$, as $w$ was occupied by $a$ at $k_a-1$ (Lemmas~\ref{lemma:traverseviasamenode} and \ref{lemma:nocollisions}). With the same argument, $c$ must have been at the goal node since at least $k_a - 2$, and if $c \neq b$, we can follow that the direct follower of $c$ in $\mathcal{B}$ is already at the goal. Continuing this down the path to $b$, $b$ must already be at the goal node at time $k_a$. Consequently, all $n_a$ agents that are not yet at the goal node are direct or indirect followers of $a$. 
    
    This means that at time $k_a$, there exists at least one direct follower $b$ that is in a node adjacent to $w$. In the next time step $k_a+1$, $b$ aims to move to $v_{prev_a} = w$. Let $\hat{b}$ be the agent that wins against all competing agents for $w$, keeps $a$ as leader (Lemma~\ref{lemma:nocollisions}) and moves there at $k_a + 1$.
    Then, in the next time step $k_a + 2$, as the goal is adjacent, $\hat{b}$ will move to $g$ and consequently be there at time $k_a+2$ while $n_a - 1$ agents are not at the goal at $k_a + 2$. Using the induction assumption, it follows that all agents reach the goal in at most $(k_a + 2) + 2(n_a-1) = k_a + 2((n_a+1)-1)$ time steps.

    Hence, as the head agent reaches $g$ at time $k_g$, it follows that all $n$ agents reaches the goal in at most $k_g + 2(n-1)$ time steps.
\end{proof}

\begin{theorem}
    Let the full-knowledge (FK) strategy  be the strategy where agents have full prior knowledge of the environment. If the used single-agent maze solver is complete, the ratio $R(n) = {M_{Algo}(n)}/{M_{FK}(n)}$ of the makespan of the proposed algorithm, $M_{Algo}(n)$, relative to the one of the FK strategy, $M_{FK}(n)$, is either $R(n)=1$ or strictly decreasing with asymptotic equivalence, $\lim_{n \to \infty} R(n) = 1$.
\end{theorem}
\begin{proof}
    If the used single-agent maze solver is complete, it finds a node adjacent to the goal at time $k_g-1$. Per Theorem~\ref{theorem:correctness}, $n$ agents executing Algorithm~\ref{mainalgorithm} reach the goal at time $M_{Algo}(n) = k_g+2(n-1)$ at the latest.
    
    If start and goal are adjacent, using Algorithm~\ref{mainalgorithm}, at every time step, one agent reaches the goal, and all agents will have reached it at time $M_{Algo}(n)=n$. Having full knowledge of the environment, agents executing the full-knowledge strategy also move one-by-one to the goal, such that all agents will have reached the goal at time $M_{FK}(n) = n$. Hence, $R(n) = M_{Algo}(n) / M_{FK}(n) = 1$.
    
    Now assume that $s$ and $g$ are not adjacent. 
    As $s$ is a leaf and the full-knowledge strategy needs to have no vertex or following conflicts at nodes $\neq g$, an agent can leave the start node only every two time steps. 
    This means that, even when following shortest paths, an agent can move to $g$ only every two time steps. Consequently, $M_{FK}(n) = k_{FK} + 2(n - 1)$, where $k_{FK}$ describes the length of the shortest path from $s$ to $g$, as the first agent that left $s$ reaches $g$ at $k_{FK}$ and every two time steps later, the next agent can reach the goal. If $k_g=k_{FK}$, $R(n)=1$. If $k_g > k_{FK}$, it follows $\left(\partial_n R\right)(n) = -2\frac{k_g - k_{FK}}{(k_{FK}+2(n-1))^2} < 0$ and
    \begin{equation*}
        \lim_{n \to \infty} R(n) = \lim_{n \to \infty} \frac{k_g+2(n-1)}{k_{FK} + 2(n - 1)} = \lim_{n \to \infty} 
        \frac{k_g/(n-1)+2}{k_{FK}/(n - 1) + 2} = 1.
    \end{equation*}
    Hence, the makespan of Algorithm~\ref{mainalgorithm} is either equal or asymptotically equivalent to the one of the full-knowledge strategy with respect to 
    $n$.
\end{proof}

\subsection{Time and Space Complexity Analysis}
\label{subsec:complexity}
Let $d \coloneq \deg \mathcal{G}$ be the maximum degree of the graph.

\emph{Space Complexity.} In the worst case, an agent $a$ is at a node with degree $d$, and all adjacent nodes are unoccupied, have degree $d$ and have active agents occupying adjacent nodes.
This results in at most $\min(n, d(d-1))$ active agents in communication range of $a$, and requires $\mathcal{O}(\min(n, d^2))$ space to store $\mathcal{C}$ and information sent by nearby agents. Agents must also maintain the \texttt{NodesTowards} list for each agent within communication range. In the worst case, an agent in a non-adjacent node sends a broadcast that reaches $a$ via all $d$ adjacent nodes, 
resulting in a worst-case size of $\mathcal{O}(d)$ for any \texttt{NodesTowards}. Agents also require $\mathcal{O}(d)$ space for $\mathcal{N}^{\text{occupied}}$, and $\mathcal{O}(1)$ for any other variables. Our proposed algorithm therefore requires $\mathcal{O}(d\min(n,d^2))$ space. Additional spatial requirements depend on the single-agent maze solver selected (e.g., BFS has a requirement of $\mathcal{O}(\big| \mathcal{V}\big| )$~\cite{even2011graph}, resulting in $\mathcal{O}(\big| \mathcal{V} \big| + d\min(n, d^2))$, whereas Tr\'emaux's algorithm without marking the environment requires $\mathcal{O}(\big| \mathcal{E}\big| )$~\cite{even2011graph}, giving a total of $\mathcal{O}(\big| \mathcal{E} \big| + d\min(n, d^2))$).

\emph{Time Complexity.} Each iteration of the algorithm has a constant number of function calls that have at maximum a
runtime of $\mathcal{O}(\min(n, d^2))$. 
Per Theorem~\ref{theorem:correctness}, the overall number of iterations depends on the single-agent maze solver and number of agents. Trémaux's algorithm takes $\mathcal{O}(\big| \mathcal{E} \big|)$ as all edges are traversed at most twice~\cite{even2011graph}, BFS needs $\mathcal{O}((\big| \mathcal{V} \big| + \big| \mathcal{E} \big|)^2)$~\cite{even2011graph} as a physical agent needs to backtrack at every visited node, and a uniform random walk is expected to run in $\mathcal{O}(\big| \mathcal{V} \big|^3)$~\cite{randomwalkruntime}.
In total, the algorithm with Tr\'emaux's algorithm needs $\mathcal{O}((\big| \mathcal{E} \big|+n)\min(n, d^2))$, with BFS $\mathcal{O}(((\big| \mathcal{V} \big| + \big| \mathcal{E} \big|)^2+n)\min(n, d^2))$ and with a uniform random walk an expected $\mathcal{O}((\big| \mathcal{V} \big|^3 + n) \min(n, d^2))$.

\section{Results}
\label{sec:results}
\begin{figure}[t]
    \centering
    \includegraphics[width=0.9\linewidth]{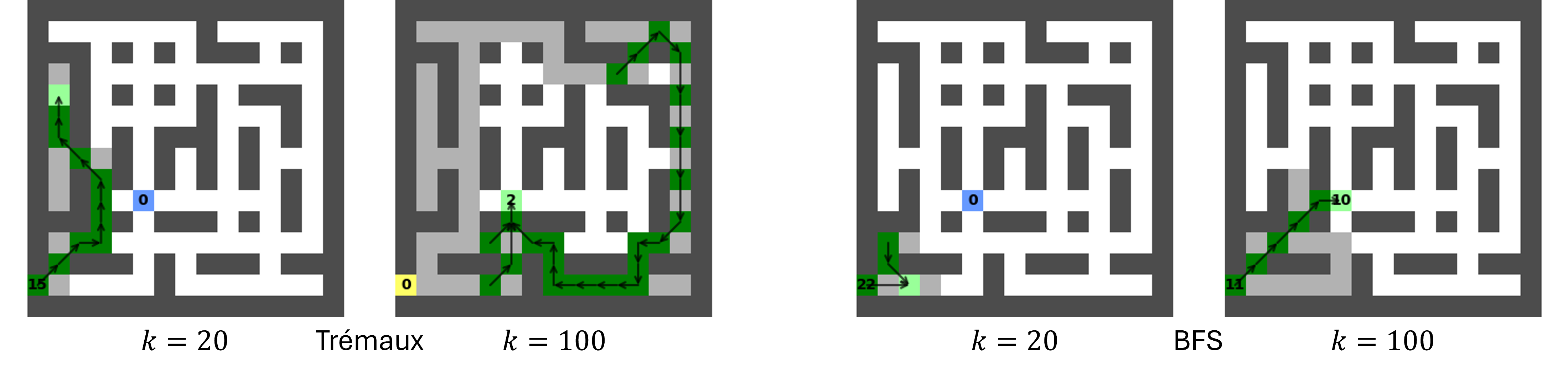}
    \caption{Simulations in grid mazes using Tr\'emaux's algorithm and BFS at times $k=20$ and $k=100$. Walls are dark gray, and nodes that were already visited are light gray. In this example, Tr\'emaux's algorithm explores a larger portion of the graph, though BFS is the first to reach the goal.}
    \label{fig:simulations}
\end{figure}

Our algorithm (hereafter, MAMT) requires a single-agent maze solver to determine head movement. We evaluated MAMT with three such solvers:
Tr\'emaux's algorithm, which explores the graph in a DFS-manner, ensuring no edge is visited more than twice; BFS, which explores the graph with increasing distance from the start node; and a uniform random walk. Simulations were conducted in square grid mazes (see Fig.~\ref{fig:simulations}).
For every combination of maze size $l \times l,l \in \{5, 15, 25, 35\}$ and number of agents $n \in \left\{1, 5, 25, 125, 625\right\}$, a same set of 20 random mazes were simulated, with a $10^4$ steps timeout. Fig.~\ref{fig:singleagentmazesolverevaluation} shows the makespan and sum-of-fuels performance. All trials were collision-free, and, for MAMT with Tr\'emaux's algorithm or BFS, all agents reached the goal. In ten trials (two for $l=25$ and eight for $l=35)$, MAMT with random walk timed out.
Although in contrast to BFS, Tr\'emaux’s algorithm does not necessarily find a shortest path, its performance was superior to that of BFS.
This result arises because Tr\'emaux’s algorithm backtracks less extensively than BFS.
As the group size increases, both makespan and average sum-of-fuels showed a tendency of approaching the mean performance of a full-knowledge strategy (i.e., omniscient agents, following a shortest path). 
Once the goal is found by the first agent, further agents follow at a steady rate, 
reducing average sum-of-fuel and the increase in makespan as the number of agents grows.

\begin{figure}[t]
    \centering
    \includegraphics[width=\linewidth]{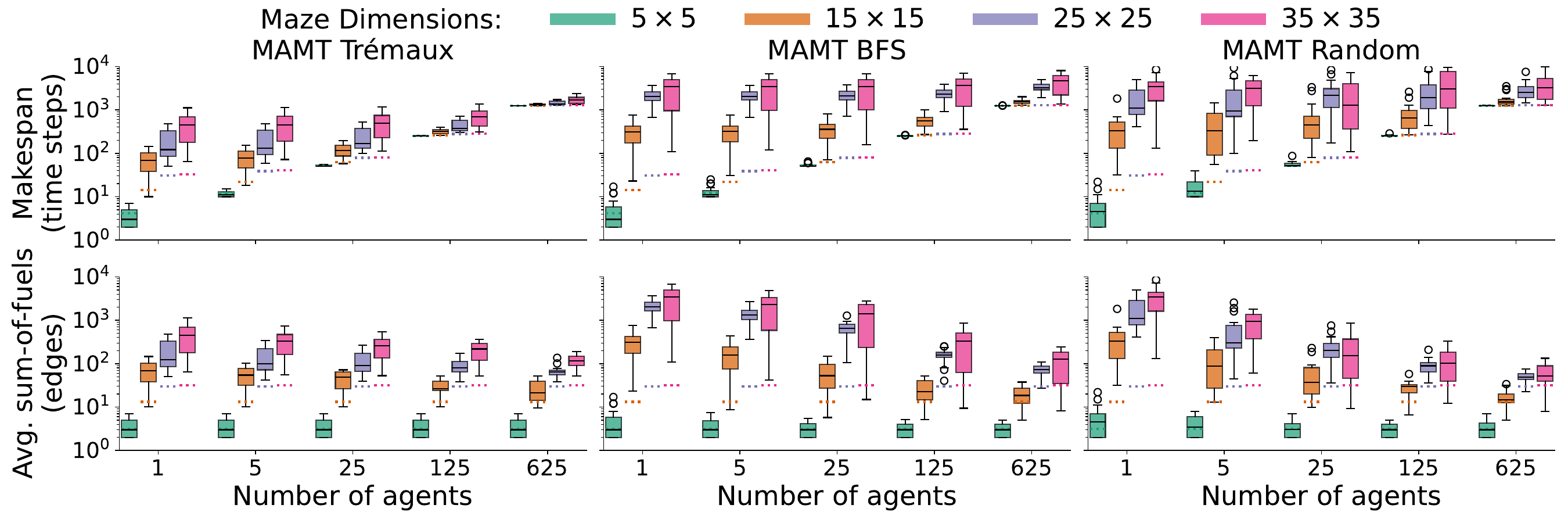}
    \caption{Makespan and average sum-of-fuel for the proposed algorithm using Tr\'emaux's algorithm, BFS, and uniform random walk 
    across various mazes and number of agents. 
    Dashed lines indicate the mean values for agents following the shortest path.}
    \label{fig:singleagentmazesolverevaluation}
\end{figure}

We evaluate MAMT against the full-knowledge strategy and a naïve baseline. For the latter, every agent $a$ independently runs the single agent maze solver to determine its next node $x_a$. The naïve baseline uses the following rules:
(i) If $x_a$ is currently unoccupied and no other agents want to move there, $a$ moves to this node;
    (ii) if $x_a$ is currently unoccupied and at least one other agent wants to move there, the agent of minimum $ID$ moves there while other competing agents wait;
    (iii) A \emph{target-cycle} occurs if it exists a subset of agents $\mathcal{B} \subseteq \mathcal{A}$ and a permutation $\pi: \mathcal{B} \mapsto \mathcal{B}$ with $x_b = v_{\pi(b)}$ $\forall b \in \mathcal{B}$. If such a target-cycle occurs and $a$ is part of it, it forwards its single-agent maze solver instance to the agent that occupies its target node (and receives an instance of another agent);
    (iv)
    if $x_a$ is currently occupied and $a$ is not part of a target-cycle, it waits.
\begin{figure}[t]
    \centering
    \includegraphics[width=\linewidth]{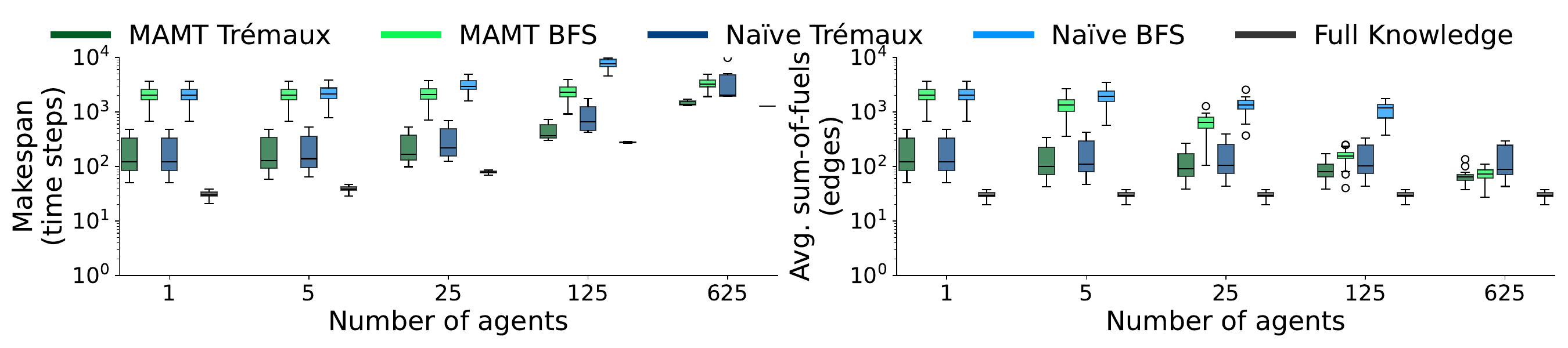}
    \caption{
    Comparison of makespan and average sum-of-fuel for MAMT and naïve strategies (each using the Tr\'emaux’s algorithm and BFS) against the full-knowledge strategy.
    Trials were conducted for up to $625$ agents in $20$ mazes of size $25 \times 25$. Trials that triggered the timeout of $10^4$ time steps were excluded.}
    \label{fig:evaluationvsnaive}
\end{figure}
Trials were performed in mazes of size $25 \times 25$ with the remaining setup unchanged.
The results are shown in Fig.~\ref{fig:evaluationvsnaive}.
The MAMT variants succeeded in all trials.
For 125 agents, the naïve strategy using BFS timed out in 9 trials. For 625 agents, the naïve strategy with Tr\'emaux’s algorithm timed out in 2 trials, whereas the BFS-based variant timed out in all 20 trials
. As before, the superior performance between Tr\'emaux’s algorithm compared to BFS is highly visible. The results show that with growing agent group size and the same underlying single-agent maze-solver, MAMT outperforms the naïve strategy in both makespan and average sum-of-fuels. 
They suggest that MAMT approaches in makespan
the full-knowledge strategy as number of agents increases.
The source code is available at \cite{RauGit}. 

\section{Conclusion}
\label{sec:conclusion}
We present an algorithm for multi-agent maze traversal, extending~\cite{Jahir} to general graphs. 
We prove that if the underlying single-agent maze solver is complete, the algorithm enables all agents to reach the undisclosed goal in finite time. We prove that achieved makespan is asymptotically equivalent to that achieved by optimal, omniscient agents. 
We show that MAMT outperforms a naïve baseline where all agents use a single-agent maze solver. We analyze the efficacy of several single-agent search algorithms when used by MAMT. Future work includes real-world experiments~\cite{Jahir} and robustness to communication failures.

\begin{credits}
\subsubsection{\ackname} 
This work was funded by BMFTR (Robotics Institute Germany; grant 16ME1001); EU Horizon Europe Framework Programme (“OpenSwarm”; grant 101093046); LOEWE center emergenCITY [LOEWE/1/12/519/03/05.001(0016)/72].
ChatGPT (OpenAI, Free Plan) assisted generation of code used for generating maze environments for simulation purposes, and code to plot Figures~\ref{fig:simulations}--\ref{fig:evaluationvsnaive}. It was also used for improving the language of the manuscript. It was not used in the design or implementation of the presented algorithm, baseline methods, or proofs.
\subsubsection{\discintname}
The authors have no competing interests to declare.
\end{credits}
%
%
\clearpage
\bibliographystyle{splncs04}
\bibliography{references}

\clearpage
\appendix
\section{Proofs of Lemmas~\ref{lemma:leaderincommunicationrange}--\ref{lemma:leadershipistree}}
\label{sec:proofs}
In the following, the full proofs of Lemmas~\ref{lemma:leaderincommunicationrange}--\ref{lemma:leadershipistree} are presented.

\subsection{Proof of Lemma~\ref{lemma:leaderincommunicationrange}}
\begin{proof}
     Proof by induction. 
    At $k=0$, there is only one active agent. It assumes the head role and the lemma is true. Let $k_0$ be the first discrete time where the head is not in the start node anymore, meaning it left the start node in time step $k_0$ and is consequently in an adjacent node. As soon as the head left the start node, a new agent $a$ that is not the head became active. As this agent then received a status message from the head, $a$ chose the head as its leader. As the head is adjacent, $a$ has its leader in communication range.

    $k \to k+1$: If $a$ was just activated at the start node in time step $k+1$, it has $L_a$ in communication range (same argument as above) and $v_{L_a}[k+1]$ is adjacent. 
    Let $h$ be the head agent at time $k$. 
    If it remains the head in time step $k+1$, it has no leader at $k+1$. Otherwise,
    there was another active agent $L$ that occupied an adjacent node of $h$ at $k$. $L$ was then chosen to be the new head and $h$'s leader. As $h$ did not leave its current node, $v_h[k] = v_h[k+1]$. In the same time step, $L$ received the head message stating that $L$ should assume the head role, and consequently did, remaining in its current node, $v_L[k] = v_L[k+1]$. Hence, $L$ is adjacent to $h$ and in $h$'s communication range at time $k+1$. 

    Let $a$ be an active agent that does not assume the head role in time step $k+1$ and was activated in an earlier time step. Per induction assumption, $a$ had its leader ${L_a}$ in communication range and at least one of the three cases is true at $k$.
    If $v_a = v_L = g$, neither $a$ nor $L$ will move in time step $k+1$ and the lemma is true for $k+1$.
    
    Let $v_a \neq g$ at $k$. As $a$ is not the head, it will execute 
    Algorithm~\ref{alg:ResolveConflict}.
    If the head $h$ is not yet at the goal and in communication range of $a$, $a$ will wait 
    for the head message from $h$. As $a$ will not assume the head role, it will check if it received the head's directed cast or if the head is already in an adjacent node and not $a$'s current leader. If so, $a$ chooses $h$ as its new leader and remains in its current node. 
    If the directed cast was received by $a$, $D^{\text{solver}}$ is an adjacent node of $v_a$ where the head will move to in this time step. 
    If the head is already in an adjacent node $v_h[k]$, it can only move to a node that is at most two edges away from $a$. 
    In this case, no other agent will move to $v_h[k]$ in time step $k+1$ (Lemma~\ref{lemma:nocollisions}).
    Hence, $a$'s new leader $h$ is adjacent to $a$ or two nodes away with the unoccupied node $v_{prev_h}[k+1]=v_h[k]$ adjacent to $v_a[k+1]$ and $v_h[k+1]$, meaning that $h \in \mathcal{C}_a$ at time $k+1$.
    If the head is not in communication range or already at the goal, or no directed cast was received and (i) the head is not adjacent or (ii) already $a$'s leader,
    $a$ checks if the goal is adjacent. If so, the head has been at the goal for multiple time steps (Lemmas~\ref{lemma:traverseviasamenode} and \ref{lemma:nocollisions}) and $a$ moves there in timestep $k+1$. 
    As $a$ must have moved to this adjacent node of the goal in the time step $k$ (else it would have already moved to the goal in a previous time step), $a$ followed its leader $L_a$ there, meaning that $L_a$ already resides at the goal, i.e., $v_a[k+1]=v_{L_a}[k+1]=g$.
    If the goal is not adjacent to $a$, $a$ checks if $L_a$ resides in an adjacent node at time $k$. If so, $a$ will keep $L_a$ as leader and not move. Either $L_a$ also remains at its current node (i.e., at time $k+1$ $a$ is adjacent to its leader $L_a$) or it moves to an adjacent node and no other agent moves to $v_{L_a}[k]$ (Lemma~\ref{lemma:nocollisions}) meaning that $L_a$ is in communication range and its previous node $v^{prev}_L[k+1]=v_L[k]$ is unoccupied and adjacent to $a$ and $L_a$ at time $k+1$.
    
    If none of the above applies (i.e., at time $k$ the leader is two nodes away from $a$ and the leader's previous node $v^{prev}_{L_a}[k]$ is unoccupied and adjacent to $v_a[k]$ and $v_{L_a}[k]$), $a$ calculates all competing agents for $v^{prev}_{L_a}[k]$ (Lemma~\ref{lemma:nocollisions}). If $a$ is chosen as the winner, it keeps its leader and moves to $v^{prev}_{L_a}[k]$. As the leader remains in its current node or moves to an adjacent node, at time $k+1$, either $L_a$ is adjacent to $a$ or two nodes away with $v^{prev}_{L_a}[k+1]=v_{L_a}[k]$ which is adjacent to $a$ (as $v_a[k+1] = v^{prev}_{L_a}[k]$) and no other agent moves to $v_{L_a}[k]$ (Lemma~\ref{lemma:nocollisions}), i.e., $L$ is in communication range. If $a$ is not selected as winner, it remains in its current node and choses the winning agent $b$ (which moves to $v^{prev}_{L_a}[k]$) as new leader. hence, at time $k+1$, $a$ is adjacent to its new leader $b$ and $b$ is in $a$'s communication range.
\end{proof}

\subsection{Proof of Lemma~\ref{lemma:leadershipistree}}

\begin{proof}
    Per Lemma~\ref{lemma:onlyonehead}, at all discrete times $k$ only one agent has no leader. 
    Every other agent has one leader, meaning that $\big|\mathcal{E}_L[k]\big| = \big|\mathcal{A}^{\text{active}}[k]\big| - 1$. Hence, if $\mathcal{T}_L[k]$ is connected, the underlying undirected graph must be a tree~\cite{mehlhorn2008algorithms}, i.e. there is never more than one path from any other node to the root node.
    
    Proof by induction. At $k=0$, only one agent is active and assumes the head role, the lemma is true. Let $k_0$ be as in Lemma~\ref{lemma:leaderincommunicationrange}; after the head leaves $s$, the newly activated agent chooses the head as its leader. 
    As the head moved to an adjacent node of $s$ in time step $k_0$, it will not choose the newly activated agent as leader (Lemma~\ref{lemma:onlyonehead}) and the lemma is true.
    
    $k \to k+1$: Assume that at time $k$, $\mathcal{T}_L[k]$ is an in-tree.
    There are four cases where an agent can change its leader in time step $k+1$. We interpret every case as operation on $\mathcal{T}_L[k]$ and show that none change the in-tree property: 
    (i) The head chooses another agent $b$ to be the new head and its leader (Algorithm~\ref{mainalgorithm} line~\ref{algo:choosenewhead}).
    (ii) A non-head agent $a$ chooses the current head as its new leader (Algorithm~\ref{alg:ResolveConflict} line~\ref{algo2:headadjacentorunicast}).
    (iii) A non-head agent $a$ competes with other non-head agents for the previous node of its leader and chooses the winner $b$ as new leader (Algorithm~\ref{alg:ResolveConflict} line~\ref{alg2:updateL}).
    (iv) An inactive agent $a$ activates at the start node $s$ and chooses the agent $b$ that just left $s$ as its leader.

    Cases (i) and (ii) can happen simultaneously at some time $k+ \epsilon$, $\epsilon \in (0,1)$.
    In case (i), the edge $(h,b)$ was added to $\mathcal{E}_L$ before $h$ sent out the head-message. $b$ will set its leader to $\texttt{nil}$ at $k+\epsilon$ after receiving the head-message, removing the edge $(b,L_b)$ from $b$ to its old leader $L_b$.
    Hence, $b$ has no outgoing edges anymore and becomes the new root. 
    For an agent $a \in \mathcal{A}^{\text{active}}[k]\setminus \{b,h\}$ that was a 
    direct or indirect follower
    from $b$ at time $k$, i.e., the unique path to $h$ was $P(a,h)[k] = P(a,b)[k] \cup \{(b,L_b)\} \cup P(L_b,h)[k]$, $P(a,b)[k]=P(a,b)[k+\epsilon]$ was not changed and is now the unique path to the new head $b$.
    For $h$, the unique path to the new head $b$ is defined by $P(h,b)[k+\epsilon] = \{(h,b)\}$.
    For an agent $a\neq b,h$ that was a descendant~\cite{mehlhorn2008algorithms} from $h$ but not $b$ at time $k$, i.e., $(b,L_b) \notin P(a,h)[k]$, the unique path to the new head $b$ results from $P(a,b)[k] = P(a,h) \cup \{(h,b)\}$.
    In case (ii), the edge from $a$ to its old leader is removed and the edge $(a,h)$ is added. If $h$ keeps the head role, the unique path from $a$ to $h$ evaluates to $P(a,h)[k+\epsilon]=\{(a,h)\}$. If $b$ is chosen as new head, the path to $b$ results to $P(a,b)[k+\epsilon]=\{(a,h),(h,b)\}$ (see (i)).
    Hence, $\mathcal{T}_L[k+\epsilon]$ is an in-tree. 
    
    At a later time $k+\delta$, $\delta \in (\epsilon,1)$ case (iii) may occur.
    In case (iii), as $b$ will keep its current leader in time step $k+1$ (Algorithm~\ref{alg:ResolveConflict} line~\ref{alg2:keepL}, Lemma~\ref{lemma:nocollisions}), the unique path to the current head $h$ results from $P(a,h)[k+\delta] = \{(a,b)\} \cup P(b,h)[k+\delta]$.
    Even if the same update occurs along the leader chain of $b$, with the same argumentation there will be a unique path $P(b,h)[k+\delta]$ from $b$ to $h$, i.e., $\mathcal{T}_L[k+\delta]$ is an in-tree.

    After moving, at time $k + \gamma$, $\gamma \in (\delta,1)$ case (iv) may occur and $\mathcal{A}^{\text{active}}[k+\gamma] = \mathcal{A}^{\text{active}}[k+\delta]\cup \{a\}$. As agent $a$ chooses $b$ as its leader, i.e., adding edge $(a,b)$, the unique path to the current head $h$ is defined by $P(a,h)[k+\gamma]=\{(a,b)\} \cup P(b,h)[k+\gamma]$.
    As no other leadership changes happen, at $k+1$, $\mathcal{T}_L[k+1] = \mathcal{T}_L[k+\gamma]$ is an in-tree.
\end{proof}
\end{document}